\documentclass[11pt]{article}

\usepackage[utf8]{inputenc} % allow utf-8 input
\usepackage[T1]{fontenc}    % use 8-bit T1 fonts
\usepackage{hyperref}       % hyperlinks
\usepackage{url}            % simple URL typesetting
\usepackage{booktabs}       % professional-quality tables
\usepackage{amsfonts}       % blackboard math symbols
\usepackage{nicefrac}       % compact symbols for 1/2, etc.
\usepackage{microtype}      % microtypography

\usepackage{amstext}
\usepackage{amsmath}
\usepackage{amsthm}
\usepackage{amssymb}
\usepackage{bm}
\usepackage{wrapfig}
\usepackage{amsfonts, amsmath}
\usepackage[utf8]{inputenc}
\usepackage{graphicx}
\usepackage{bbm, comment}
\usepackage{subfigure}
\usepackage{color}
\usepackage{float}
\usepackage{wrapfig}
\usepackage{enumitem}
\usepackage{url}
\usepackage{MnSymbol,wasysym,marvosym,ifsym}

\usepackage{algorithm}
\usepackage[noend]{algpseudocode}

\makeatletter
\def\BState{\State\hskip-\ALG@thistlm}
\makeatother

\usepackage{tabulary}
\usepackage{booktabs}

\usepackage[top=1.in, bottom=1.in, left=1.in, right=1.in]{geometry}
\usepackage[numbers]{natbib}

\newtheorem*{theorem*}{Theorem}
\newtheorem{theorem}{Theorem}[section]
\newtheorem{lemma}[theorem]{Lemma}

\newtheorem{claim}[theorem]{Claim}
\newtheorem{assumption}[theorem]{Assumption}
\newtheorem{remark}[theorem]{Remark}

\newtheorem{example}[theorem]{Example}

\newtheorem{definition}[theorem]{Definition}

\newcommand{\E}{\mathbb{E}}

\newcommand\numberthis{\addtocounter{equation}{1}\tag{\theequation}}

\usepackage{xcolor}
\definecolor{cyan2}{HTML}{1B8988}

\newcommand\gs[1]{{}}
\newcommand\yk[1]{{}}

\allowdisplaybreaks

\title{Water from Two Rocks: Maximizing the Mutual Information}

\title{Water from Two Rocks: Maximizing the Mutual Information} 

\author{Yuqing Kong\\ University of Michigan \and Grant Schoenebeck\\ University of Michigan}
\date{}

\begin{document}

% End generated code
%

%\keywords{Peer prediction, co-training, information theory}

\maketitle

\begin{abstract}
We build a natural connection between the learning problem, co-training, and forecast elicitation without verification (related to peer-prediction)
 and address them simultaneously using the same information theoretic approach.\footnote{This work is supported by the National Science Foundation, under grant CAREER\#1452915, CCF\#1618187 and AitF\#1535912.}

In co-training/multiview learning~\citep{blum1998combining} the goal is to aggregate two views of data into a prediction for a latent label.  We show how to optimally combine two views of  data by reducing the problem to an optimization problem.  
Our work gives a unified and rigorous approach to the general setting.   

In forecast elicitation without verification we seek to design a mechanism that elicits high quality forecasts from agents  in the setting where the mechanism does not have access to the ground truth.  By assuming the agents' information is independent conditioning on the outcome, we propose mechanisms where truth-telling is a strict equilibrium for  both the single-task and multi-task settings.  Our multi-task mechanism additionally has the property that the truth-telling equilibrium pays better than any other strategy profile and strictly better than any other ``non-permutation" strategy profile when the prior satisfies some mild conditions.

\end{abstract}

\section{Introduction}\label{sec:intro}

Co-training/multiview learning is a problem that asks to aggregate two views of data into a prediction for the latent label, and was first proposed by \citet{blum1998combining}. Although co-training is an important learning problem, it lacks a unified and rigorous approach to the general setting. The current paper will make an innovative connection between the co-training problem and a peer prediction style mechanism design problem: forecast elicitation without verification, and develop a unified theory for both of them via the same information theoretic approach. 

We use ``forecasting whether a startup company will succeed'' as our running example. We have two possible sources of information for each startup: the features $X_A$ (e.g. products, business idea, target customer) of the startup; and the survey feedback $X_B$, collected from the crowd (e.g.\ a survey of amateur investors). Sometimes we have access to both the sources, and sometimes we have access to only one of the sources. We want to learn how to forecast the result $Y$ (succeed/fail) of a startup company, using both or one of the sources.

We are given a set predictor candidates $\{P_A\}$ (e.g. a set of hypotheses) such that each predictor candidate $P_A$ maps the features $X_A$ to a forecast for the result $Y$ of the startup (e.g. succeed with 73\% probability, fail with 27\% probability). We are also given a set predictor candidates $\{P_B\}$ (e.g. a set of aggregation algorithms like majority vote/weighted average) such that each predictor candidate $P_B$ maps the survey feedback $X_B$ to a forecast for the result $Y$. Our goal is to evaluate the performance of a specific pair $P_A,P_B$. The learning problem, learning how to forecast, can be reduced to this goal since if we know how to evaluate the two candidates $P_A,P_B$'s performance, we can select the two candidates $P_A^*,P_B^*$ which have the highest performance and use them to forecast.

Given a batch of past startup data each with the features $X_A$, the crowdsourced feedback $X_B$, and the result $Y$, we can evaluate the performance of the predictors through many existing measurements (e.g. proper scoring rules, loss functions). This evaluation method is related to the supervised learning setting. However, there may be only very few data points about the startups with results $Y$.\footnote{For example, if we focus on cryptographic or self-driving currencies, there are very few startups labeled with results.} When we only use a few labeled data points to train the predictor, the predictor will likely over-fit. Thus, we can boldly ask:

%\vspace{5pt}
(*Learning) \emph{Can we evaluate the performance of the predictor candidates, as well as learn how to forecast the ground truth $Y$, without access to any data labeled with $Y$?} (See Figure~\ref{fig:peer1})
%\vspace{5pt}
{\setlength\intextsep{0pt}
\begin{figure}[h]
\centering
\includegraphics[width=0.7\linewidth]{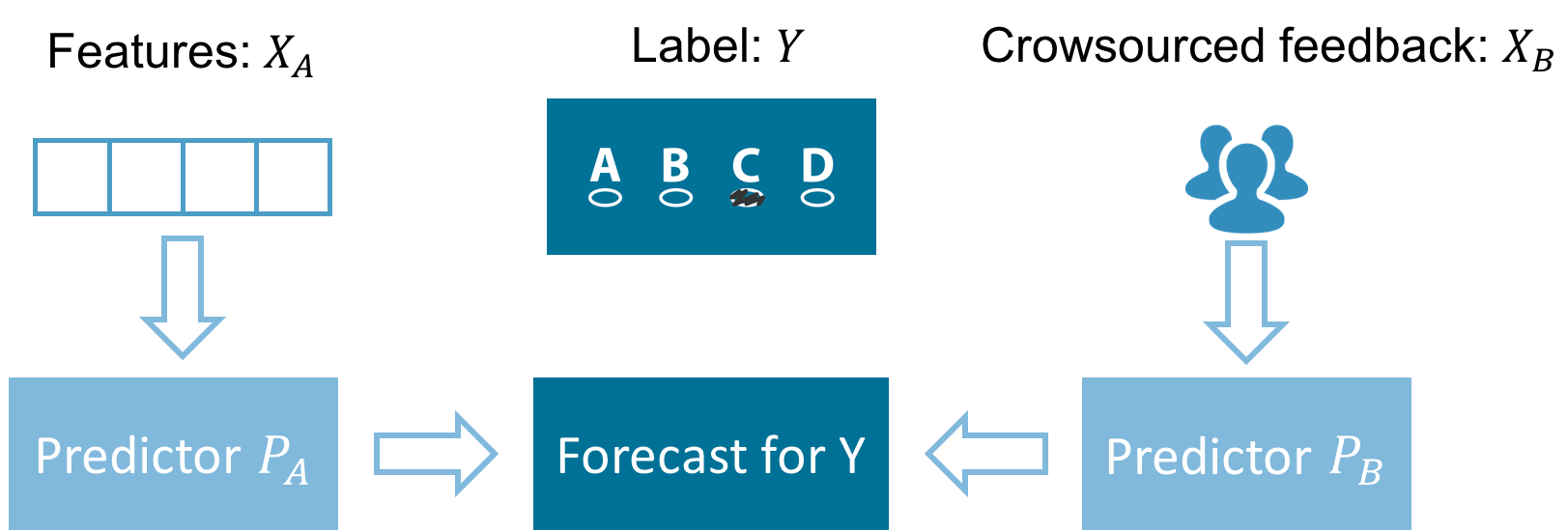}
\caption{Problem (*): Finding the common ground truth}
\label{fig:peer1}
\end{figure}}

It is impossible to solve this problem without making an additional assumption on the relationship between $X_A,X_B$ and $Y$. However, it turns out we can solve this problem with a natural assumption, conditioning on $Y$, $X_A$ and $X_B$ are independent.  This assumption states that $Y$ contains all common information between $X_A$ and $X_B$ (see Section~\ref{sec:model} for more discussion).

With this assumption, a naive approach is to learn the joint distribution of $X_A$ and $X_B$ using the past data, and then solve the relationship between $Y$ and $X_A,X_B$ by some calculations, using the fact that $X_A$ and $X_B$ are independent conditioning on $Y$. However, this naive approach will not work if either $X_A$ or $X_B$ has very high dimension. We will address this issue using learning methods. Before we go further on the learning problem, let's consider a corresponding mechanism design problem. In the scenario where the forecasts are provided by human beings, we want to ask a mechanism design problem:

%\vspace{5pt}
(**Mechanism design) \emph{Can we design proper \emph{instant} reward schemes to incentivize high quality forecast for $Y$ without instant access to $Y$?} (See Figure~\ref{fig:peer})
%\vspace{5pt}
{\setlength\intextsep{0pt}
\begin{figure}[t]
\centering
\includegraphics[width=0.7\linewidth]{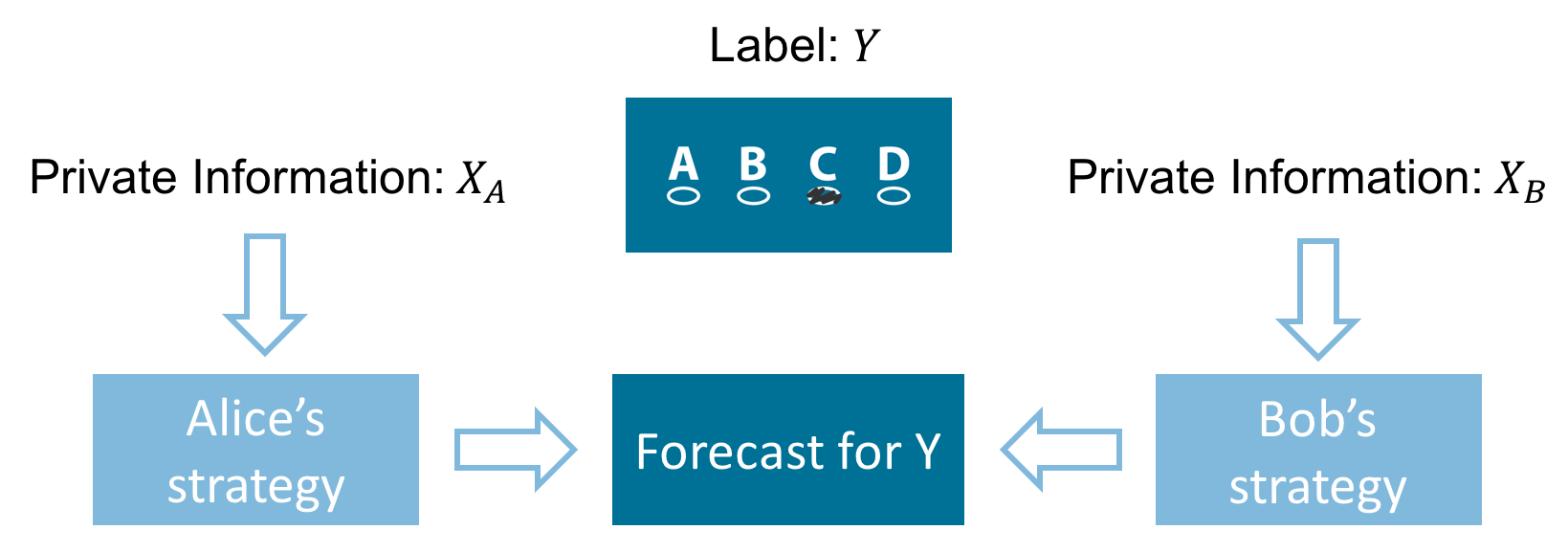}
\caption{Problem (**): Forecast elicitation}
\label{fig:peer}
\end{figure} }

People will obtain instant payments from \emph{instant} reward schemes. If we do not require the reward schemes to be instant, proper scoring rules will work by rewarding people in the future after $Y$ is revealed. It turns out the above learning problem (*) and mechanism design problem (**) are essentially the same, since there is a natural correspondence between an evaluation of their performance and their rewards. The mechanism design applications still require the conditional independent assumption. To address the two problems, a first try would be rewarding the predictors according to their ``agreement'', since high quality predictors should have a lot of agreement with each other. However, if we train the predictors based on this criterion, then the output of the training process will be two meaningless constant predictors which perfectly agree with each other (e.g. always forecast 100\% success). We call this problem the ``naive agreement'' issue.

Note that the mechanism design problem (**) is closely related to the peer prediction literature, incentivizing high quality information reports without verification. It is natural to leverage the techniques and insights from peer prediction to address problems (*) and (**). In fact, the peer prediction literature provides an information theoretic idea to address the ``naive agreement'' issue, that is, replacing ``agreement'' by mutual information. In the current paper, we will show that with a natural assumption, conditioning on $Y$, $X_A$, and $X_B$ are independent, we can address problem (*) and (**) simultaneously via rewarding the predictors the mutual information between them and using the predictors' reward as the evaluation of their performance.

\paragraph{Our contribution} We build a natural connection between mechanism design and machine learning by simultaneously addressing a learning problem and a mechanism design problem in the context where ground truth is unknown, via the same information theoretic approach. %We show that building the innovative connections among game theory, information theory, and learning theory will make new progress in all areas.

\begin{description}
    \item [Learning] We focus on the co-training problem~\cite{blum1998combining}: learning how to forecast $Y$ using two sources of information $X_A$ and $X_B$, without access to any data labeled with ground truth $Y$ (Section~\ref{sec:model}). By making a typical assumption in the co-training literature, conditioning on $Y$, $X_A$ and $X_B$ are independent, we reduce the learning problem to an optimization problem $\max_{P_A,P_B}MIG^f(P_A,P_B)$ such that solving the learning problem is equivalent to picking the $P_A^*,P_B^*$ that maximize $MIG^f(P_A,P_B)$, i.e., the $f$-mutual information gain between $P_A$ and $P_B$ (Section~\ref{sec:commontruth}). Formally, we define \emph{the Bayesian posterior predictor} as the predictor that maps any input information $X=x$ to its Bayesian posterior forecast for $Y=y$, i.e., $Pr(Y=y|X=x)$. Then when both $P_A,P_B$ are Bayesian posterior predictors, $MIG^f(P_A,P_B)$ is maximized and the maximal value is the $f$-mutual information between $X_A$ and $X_B$. With an additional mild restriction on the prior, $MIG^f(P_A,P_B)$ is maximized if and only if both $P_A,P_B$ are permuted versions of the Bayesian posterior predictor. 
    
    We also design another family of optimization goals, \emph{$PS$-gain}\footnote{$PS$ is a proper scoring rule.}, based on the family of proper scoring rules (Section~\ref{sec:psgain}). We can also reduce the learning problem to the $PS$-gain optimization problem. We will show a special case of the $PS$-gain, picking $PS$ as the logarithmic scoring rule $LSR$, corresponds to the maximum likelihood estimator method. The range of applications of $PS$-gain is more limited when compared with the range of applications of the $f$-mutual information gain, since the application of $PS$-gain requires either one of the information sources to be low dimensional or that we have a simple generative model for the distribution over one of the information sources and ground truth labels, while the $f$-mutual information gain does not have these restrictions.
    
    As is typical in related literature, we do not investigate the computation complexity or data requirement of the learning problem. 
    
    \emph{To the best of our knowledge, this is the first optimization goal in the co-training literature that guarantees that the maximizer corresponds to the Bayesian posterior predictor, without any additional assumption}. Thus, our method optimally aggregates the two sources of information.  
    
    \vspace{5pt}
    
    \item [Mechanism design] Consider the scenario where 
    we elicit forecasts for ground truth $Y$ from agents and pay agents immediately. Without access to $Y$, given the prior on the distribution of $Y$, i.e., $Pr[Y]$, \footnote{This is not a very strong assumption since we do not need the knowledge of the joint distribution over the event and agents' private information.} by assuming agents' private information are independent conditioning on $Y$ and the prior satisfies some mild conditions, in the single-task setting (there is only a single forecasting task), we design a \emph{strictly truthful} mechanism, the \emph{common ground mechanism}, where truth-telling is a strict equilibrium (Section~\ref{sec:single}); in the multi-task (there are at least two a priori similar forecasting tasks) setting, we design a family of \emph{focal} mechanisms, the \emph{multi-task common ground mechanism $MCG(f)$s}, where the truth-telling equilibrium pays better than any other strategy profile and \emph{strictly} higher than any non-permutation strategy profile (Section~\ref{sec:multi}). 
\end{description}

%%%overlapping defined before

\paragraph{Technical contribution}Our main technical ingredient is a novel performance measurement, the \emph{$f$-mutual information gain}, which is an unbiased estimator of the $f$-mutual information. To give a flavor of this measurement, we give an informal presentation here: both $P_A$ and $P_B$ are assigned a batch of forecasting tasks, the $f$-mutual information gain between $P_A$ and $P_B$ is
\begin{align*}
    &\text{The agreements between $P_A$'s forecast and $P_B$'s forecast for the same task} \\
    &- f^{\star}(\text{The agreements between $P_A$'s forecast and $P_B$'s forecast for different tasks})
\end{align*}

{\setlength\intextsep{5pt}
\begin{figure}[htp]
\centering
\includegraphics[width=0.7\linewidth]{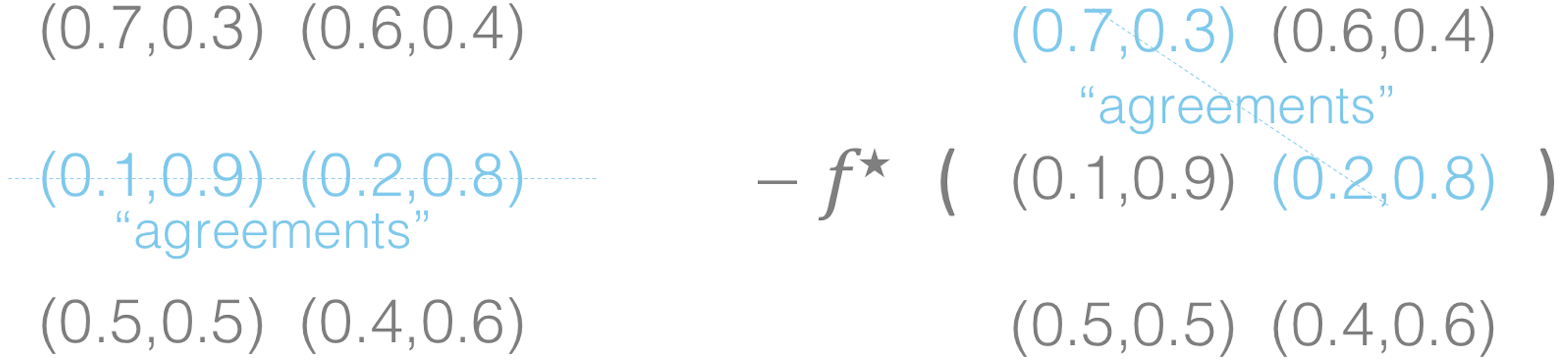}
\caption{An unbiased estimator of $f$-mutual information: $f$-mutual information gain. $P_A$ and $P_B$ are assigned three forecasting tasks. $P_A$'s outputs are $(0.7,0.3),(0.1,0.9),(0.5,0.5)$ and $P_B$'s outputs are $(0.6,0.4),(0.2,0.8),(0.4,0.6)$. To calculate the $f$-mutual information gain between them, we pick a task (e.g. Task no.\ 2) uniformly at random and calculate the ``agreement" $a_s$ between $P_A$ and $P_B$'s forecasts for this task; we also pick a pair of distinct tasks $(i,j)$ uniformly at random (e.g. (Task no.\ 1, Task no.\ 2)) and calculate the ``agreement" $a_d$ between $P_A$'s forecast for task $i$  and $P_B$'s forecast for this task $j$. The $f$-mutual information gain is then $a_s-f^{\star}(a_d)$. The formal definition (Section~\ref{sec:fgain}) actually uses the  empirical expectations of $a_s$ and $f^{\star}(a_d)$.}
\label{fig:fgain}
\end{figure} 
}

where $f^{\star}$ is the conjugate of the convex function $f$. With this measurement, two agreeing constant predictors have small gain since their outputs have large agreements for both the same task and different tasks. The formal definition will be introduced in Section~\ref{sec:fgain} and the agreement measure is introduced in Definition~\ref{def:agree}. 

The $f$-mutual information gain is conceptually similar to the correlation payment scheme proposed by \citet{dasgupta2013crowdsourced} (in the binary choice setting), and \citet{2016arXiv160303151S} (in the multiple choice setting), which pays agents ``the agreement for the same task \emph{minus} the agreement for the distinct task''. In \citet{dasgupta2013crowdsourced} and \citet{2016arXiv160303151S}, the payment scheme is designed for discrete signals and the measure of agreements is a simple indicator function. \citet{2016arXiv160501021K} show that this correlation payment is related to a special $f$-mutual information. Thus, the $f$-mutual information gain can be seen as an extension of the correlation payment scheme that works for forecast reports.

\subsection{Applications}\label{sec:application}

In our startup running example, we consider the situation where one source of information is the features and another source of information is the crowdsourced feedback. In fact, our results apply to all kinds of information sources. For example, we can make both sources features or crowdsourced feedback. Different setups for the information sources and predictor candidates can bring different applications of our results. 

Let's consider the ``learning with noisy labels'' problem where the labels in the training data are a noisy version of the ground truth labels $Y$ and the noise is independent. We can map this problem into our framework by letting $X_B$ be the noisy label of features $X_A$. That is, $X_B$ is a noisy version of $Y$. Our framework guarantees that the Bayesian posterior predictor that forecasts $Y$ using $X_A$ must be part of a maximizer of the optimization problem. However, there are many other maximizers. For example, since $X_A$ and $X_B$ are independent conditioning $X_B$.  The Bayesian posterior predictor that forecasts $X_B$ using $X_A$ is also part of a maximizer, since the scenario $Y=X_B$ also satisfies the conditional independence assumption. If $X_B$ has much higher dimension than $Y$, we do not have this issue. But $X_B$ has the same signal space with $Y$ in the learning with noisy label problem. Thus, it's impossible to eliminate other maximizers without any side information here. With some side information (e.g. a candidate set $\mathcal{F}$, like linear regressions, that only contains our desired maximizer.), it's possible to obtain the Bayesian posterior predictor that forecasts $Y$ using $X_A$. Note that our framework does not require a pre-estimation on the transition probability that transits the ground truth label $Y$ to the noisy ground truth label $X_B$, since our framework has this transition probability, which corresponds to the predictor $P_B$, as parameters as well and learns the correct forecaster $P_A$ and the transition probability $P_B$ simultaneously.

\citet{ratner2016data} propose a method to collect massive labels by asking the crowds to write heuristics to label the instances. Each instance is associated with many noisy labels outputted by the heuristics. In their setting, the crowds use a different source of information from the learning algorithm (e.g. the learning algorithm uses the biology description of the genes and the crowds use the scientific papers about the gene). Thus, the conditional independence assumption is natural here and we can map this setting's training problem into our framework. \citet{ratner2016data} preprocess the collected labels to approximate ground truth by assuming a particular information structure model on the crowds. Our framework is model-free and does not need to preprocess the collected labels since we can learn the best forecaster (predictor $P_A$) and the best processing/aggregation algorithm (predictor $P_B$) simultaneously.  

Moreover, since the highest evaluation value of the predictors $P_A,P_B$ is the $f$-mutual information between $X_A$ and $X_B$, our results provide a method to calculate the $f$-mutual information between any two sources of information $X_A,X_B$ of any format. \citet{2016arXiv160501021K} propose a framework for designing information elicitation mechanisms that reward truth-telling by paying each agent the $f$-mutual information between her report and her peers' report. Thus, the $f$-mutual information gain method can be combined with this framework to design information elicitation mechanisms when the information has a complicated format.

%\paragraph{Technical novelty} 

\subsection{Related work}

\paragraph{Learning}

Co-training/multiview learning was first proposed by \citet{blum1998combining} and explored by many works (e.g. \citet{dasgupta2002pac,collins1999unsupervised}). \citet{xu2013survey, li2016multi} give surveys on this literature.  
Although co-training is an important learning problem, it lacks a unified theory and a solid theoretic guarantee for the general model. Most traditional co-training methods require additional restrictions on the hypothesis space (e.g. weakly good hypotheses) to address the ``naive agreement'' issue and  fail to deal with soft hypotheses.  Soft hypotheses output a continuous signal (as opposed to hard hypothesis which output a discrete signal) and are typically required to fully aggregate the information from two sources. \citet{becker1996mutual} deals with a feature learning problem which is very similar to the co-training problem. \citet{becker1996mutual} seeks to maximize the Shannon mutual information between the output of two functions. However, their work only considers hard (not soft) hypotheses and lacks a solid theoretic analysis for the maximizer. \citet{kakade2007multi} consider the multi-view regression and maximize the correlation between the two hypotheses. Their method captures the ``mutual information'' idea (in fact, correlation is a special $f$-mutual information \cite{2016arXiv160501021K}) but their model has a very specific set up and the analysis cannot be extended to other co-training problems.

%There are several typical co-training methods: 1) \emph{consensus principle}: maximizing the agreement between two hypotheses with some restrictions on the hypothesis space; \citet{dasgupta2002pac} provide a PAC-style bound on generalization error for a consensus principle based method proposed by \citet{collins1999unsupervised}. 2) \emph{correlation principle}: maximizing the correlation between the two hypotheses; 3) \emph{boosting}: assuming one of the hypotheses is weakly good, and using the weakly good hypothesis to generate labels and training another hypothesis. Method 1) and 3) require additional assumptions. Method 2) captures the ``mutual information'' idea and in fact, correlation belongs to the family of $f$-mutual information. The above methods all fail to deal with soft hypotheses and fully aggregate the information. 

In contrast, we propose a simple, powerful and general information theoretic framework, $f$-mutual information gain, that has a solid theoretic guarantee, works for soft hypothesis and addresses the ``naive agreement'' issue without any additional assumption. 

%\citet{natarajan2013learning}, \citet{sukhbaatar2014learning} and many other works (e.g. \cite{khardon2007noise,scott2013classification}) consider the learning with noisy labels problem. \citet{natarajan2013learning} consider binary labels and calibrate the original loss function such that the Bayesian posterior predictor that forecasts ground truth $Y$ is a maximizer of the calibrated loss. \citet{sukhbaatar2014learning} extend this work to the multiclass setting. These works require additional estimation steps to learn the transition probability that transits the ground truth labels to the noisy labels and fix this transition probability in their calibration step. In contrast, by mapping this problem into our framework (Section~\ref{sec:application}), we do not need the additional estimation steps to make the calibrated forecaster part of a maximizer of our optimization problem, and can incorporate any kind of side information to learn the calibrated forecaster and true transition probability simultaneously. 

\citet{natarajan2013learning}, \citet{sukhbaatar2014learning} and many other works (e.g. \citet{angluin1988learning,khardon2007noise,scott2013classification}) consider the learning with noisy labels problem. \citet{natarajan2013learning} consider binary labels and calibrate the original loss function such that the Bayesian posterior predictor that forecasts ground truth $Y$ is a maximizer of the calibrated loss. \citet{sukhbaatar2014learning} extend this work to the multiclass setting. These works require additional estimation steps to learn the transition probability that transits the ground truth labels to the noisy labels and fix this transition probability in their calibration step. In contrast, by mapping this problem into our framework (Section~\ref{sec:application}), we do not need the additional estimation steps to make the calibrated forecaster part of a maximizer of our optimization problem, and can incorporate any kind of side information to learn the calibrated forecaster and true transition probability simultaneously. 

Moreover, our results can handle more complicated setting where each instance is labeled by multiple labels. Rather than preprocessing the labels by a particular algorithm (e.g. majority vote, weighted average, spectral method) and assuming some information structure model among the crowds \cite{ratner2016data}, our framework is model-free and can learn the best calibrated forecaster (predictor $P_A$) and the best processing algorithm (predictor $P_B$) simultaneously.

\citet{raykar2010learning} also \emph{jointly} learn the calibrated forecaster and the distribution over the crowd-sourced feedback and ground truth labels. \citet{raykar2010learning} uses the maximum likelihood estimator and assumes a simple generative model for the distribution over the crowdsourced feedback and the ground truth labels, which is conditioning the ground truth label, the crowdsourced feedback is drawn from a binomial distribution, while our framework is model-free.  We also extend the maximum likelihood estimator method in \citet{raykar2010learning} to a general family of estimators, $PS$-gain estimators, based on the family of proper scoring rules, which also \emph{jointly} learn the calibrated forecaster and the distribution. We will show the range of applications of $PS$-gain is more limited compared with the range of applications of the $f$-mutual information gain (see Section~\ref{sec:comparison} for more details). \citet{cid2012proper} also uses proper scoring rules to design the loss functions that address the learning with noisy labels problem. However, \citet{cid2012proper} designs a different family of loss functions from the $PS$-gain and cannot jointly learn the calibrated forecaster and the distribution.

Generative Adversarial Networks (GAN) \cite{goodfellow2014generative} combine game theory and learning theory to make innovative progress. We also combine game theory and learning theory by proposing a peer prediction game between two predictors. The game in GAN is a zero-sum competitive game while the game in the current paper is collaborative. %Moreover, in addition to game theory and learning theory, we also leverage the techniques and insights from information theory.  

%Several learning problems (e.g. finding the pose of an object in an image \cite{bell1995information}, blind source separation \cite{cardoso1997infomax}) use mutual information maximization (infomax) as their optimization goal. Some of these problems require data labeled with ground truth and some of them have a very different problem set up than our work.

Several learning problems (e.g. finding the pose of an object in an image \cite{bell1995information}, blind source separation \cite{cardoso1997infomax}, feature selection \cite{peng2005feature}) use mutual information maximization (infomax) as their optimization goal. Some of these problems require data labeled with ground truth and some of them have a very different problem set up than our work.

We borrow the techniques about the duality of $f$-divergence from \citet{nguyen2009surrogate,nguyen2010estimating}. \citet{nguyen2009surrogate} show a correspondence between the $f$-divergence and the surrogate loss in the \emph{binary supervised learning} setting and \citet{nguyen2010estimating} propose a way to estimate the $f$-divergence between two high dimensional random variables. We apply the duality of $f$-divergence to an unsupervised learning problem and not restricted to the binary setting.

%We also use mutual information maximization as our optimization goal but solve a very different learning problem here.

%We also differ from the crowdsourcing literature that infers ground truth answers from agents' reports (e.g. \cite{zhang2014spectral}) in the sense that their agents' reports are a simple choice (e.g. A, B, C, D) while in our setting,  the report can come from a space larger than the space of ground truth answers, perhaps even a very high dimensional vector. 

We also differ from the crowdsourcing literature that infers ground truth answers from agents' reports (e.g. \cite{zhou2012learning,karger2014budget,zhang2014spectral,dalvi2013aggregating}) in the sense that their agents' reports are a simple choice (e.g. A, B, C, D) while in our setting,  the report can come from a space larger than the space of ground truth answers, perhaps even a very high dimensional vector.

\paragraph{Mechanism design} Our mechanism design setting differ from the traditional peer prediction literature (e.g.\cite{MRZ05,prelec2004bayesian,dasgupta2013crowdsourced,2016arXiv160501021K,2016arXiv160303151S}) since we are eliciting forecast rather than a simple signal. We can discretize the forecast report and apply the traditional peer prediction literature results. However, this will only provide approximated truthfulness and fail to design focal mechanisms which pay truth-telling \emph{strictly} better than any other non-permutation equilibrium since the forecast is discretized, while our mechanisms are focal for $\geq$2 tasks setting. 

\citet{witkowski2017proper} consider the forecast elicitation situation and assume that they have an unbiased estimator of the optimal forecast while we assume an additional conditional independence assumption but do not need the unbiased estimator. 

\citet{Liu:2017:MAP:3033274.3085126,liuchen} connect mechanism design with learning by using the learning methods to design peer prediction mechanisms. In the setting where several agents are asked to label a batch of instances, \citet{Liu:2017:MAP:3033274.3085126} design a peer prediction mechanism where each agent is paid according to her answer and a reference answer generated by a classification algorithm using other agents' reports. \citet{liuchen} also use surrogate loss functions as tools to develop a multi-task mechanism that achieves truthful elicitation in dominant strategy when the mechanism designer only has access to agents' reports. %Their method treats a random peer agent's report as a noisy ground truth and applies surrogate losses to achieve truthfulness in dominant strategy. 
Instead of using learning methods to design the peer prediction mechanisms, our work uses peer prediction mechanism design techniques to address a learning problem. Moreover, our mechanism design problem has a very different set up from \citet{Liu:2017:MAP:3033274.3085126, liuchen}. \citet{agarwal2015consistent} connect learning theory with information elicitation by showing the equivalence between the calibrated surrogate losses in \emph{supervised} learning and the elicitation of certain properties of the underlying conditional label distribution. Both our learning problem and mechanism design problem have a very different set up from theirs. 

\paragraph{Independent work} Like the current paper, \citet{DBLP:journals/corr/abs-1802-07572} also uses Shannon mutual information to propose an information theoretic training objective that can deal with soft hypotheses/classifiers.  However, the optimization functions from these two works are different. 
We also use a more general information measure, $f$-mutual information, which has Shannon mutual information as a special case, and provide a formal analysis for this general framework.  Additionally, we propose an innovative connection between co-training and peer prediction.  

 %Moreover, they usually maximize the mutual information between $h(X)$ and $Y$ to figure out the best $h$ and we maximize the mutual information between $h_A(X_A)$ and $h_B(X_B)$ to figure out the best $h_A,h_B$ simultaneously. Mutual information criterion is also used in feature selection \cite{peng2005feature}, which is also different from our problem.  

%crowdsourcing
%alignment
%learning from noisy labels
%peer prediction, forecast elicitation
%mechanism design and learning

\section{Preliminaries}\label{sec:prelim}
Given a finite set $[N]:=\{1,2,...,N\}$, for any function $\phi:[N]\mapsto \mathbb{R}$, we use $(\phi(y))_{y\in[N]}$ to represent the vector $(\phi(1),\phi(2),...,\phi(N))\in \mathbb{R}^N$. Given a finite set $\Sigma$, $\Delta_{\Sigma}$ is the set of all distributions over $\Sigma$.

\subsection{$f$-divergence and Fenchel's duality}

\paragraph{$f$-divergence~\cite{ali1966general,csiszar2004information}} % We will introduce $f$-divergence---a measure for the difference between two probability distributions and its properties (\cite{amari2010information}).
$f$-divergence $D_f:\Delta_{\Sigma}\times \Delta_{\Sigma}\mapsto \mathbb{R}$ is a non-symmetric measure of the difference between distribution $\mathbf{p}\in \Delta_{\Sigma} $ and distribution $\mathbf{q}\in \Delta_{\Sigma} $ %. $f$-divergence of $\mathbf{p}$ and $\mathbf{q}$
and is defined to be $$D_f(\mathbf{p},\mathbf{q})=\sum_{\sigma\in \Sigma}
\mathbf{p}(\sigma)f\bigg( \frac{\mathbf{q}(\sigma)}{\mathbf{p}(\sigma)}\bigg)$$
where $f:\mathbb{R}\mapsto\mathbb{R}$ is a convex function and $f(1)=0$. 

Here we introduce two $f$-divergences in common use: KL divergence, and Total Variance Distance.
\begin{example}[KL divergence]
Choosing $-\log(x)$ as the convex function $f(x)$, $f$-divergence becomes KL divergence $D_{KL}(\mathbf{p},\mathbf{q})=\sum_{\sigma}\mathbf{p}(\sigma)\log\frac{\mathbf{p}(\sigma)}{\mathbf{q}(\sigma)}$
\end{example}

\begin{example}[Total Variance Distance]
Choosing $|x-1|$ as the convex function $f(x)$, $f$-divergence becomes Total Variance Distance $D_{tvd}(\mathbf{p},\mathbf{q})=\sum_{\sigma}|\mathbf{p}(\sigma)-\mathbf{q}(\sigma)|$
\end{example}

\begin{definition}[Fenchel Duality \cite{rockafellar1966extension}]
Given any function $f:\mathbb{R}\mapsto \mathbb{R}$, we define its convex conjugate $f^{\star}$ as a function that also maps $\mathbb{R}$ to $\mathbb{R}$ such that $$f^{\star}(x)=\sup_{t} tx-f(t).$$
\end{definition}

\begin{lemma}[Dual version of $f$-divergence~\cite{nguyen2009surrogate,nguyen2010estimating}]\label{lemma:dualdivergence}
$$ D_f(\mathbf{p},\mathbf{q}) \geq \sup_{u\in \Sigma} \E_{\mathbf{p}} u- \E_{\mathbf{q}}f^{\star}(u)=\sup_{u\in \mathcal{G}} \sum_{\sigma}u(\sigma) \mathbf{p}(\sigma)- \sum_{\sigma}f^{\star}(u(\sigma))\mathbf{q}(\sigma) $$
where $\mathcal{G}$ is a set of functions that maps $\Sigma$ to $\mathbb{R}$. The equality holds if and only if $u(\sigma)=u^*(\sigma)\in \partial{f}(\frac{\mathbf{p}(\sigma)}{\mathbf{q}(\sigma)})$, i.e., the subdifferential of $f$ on value $\frac{\mathbf{p}(\sigma)}{\mathbf{q}(\sigma)}$.  
\end{lemma}

We call $(u^*,f^{\star}(u^*))$ \emph{a pair of best disinguishers}.  This dual version of $f$-divergence is introduced by \citet{nguyen2009surrogate} and also plays a key role in the design of a type of generative adversarial networks, $f$-GANs~\cite{nowozin2016f}. 

\subsection{$f$-mutual information}

Given two random variables $X,Y$ whose realization space are $\Sigma_X$ and $\Sigma_Y$, let $\mathbf{U}_{X,Y}$ and $\mathbf{V}_{X,Y}$ be two probability measures where $\mathbf{U}_{X,Y}$ is the joint distribution of $(X,Y)$ and $\mathbf{V}_{X,Y}$ is the product of the marginal distributions of $X$ and $Y$. Formally, for every pair of $(x,y)\in\Sigma_X\times\Sigma_Y$, $$\mathbf{U}_{X,Y}(X=x,Y=y)=\Pr[X=x,Y=y]\qquad \mathbf{V}_{X,Y}(X=x,Y=y)=\Pr[X=x]\Pr[Y=y].$$ 

If $\mathbf{U}_{X,Y}$ is very different from $\mathbf{V}_{X,Y}$, the mutual information between $X$ and $Y$ should be high since knowing $X$ changes the belief for $Y$ a lot. If $\mathbf{U}_{X,Y}$ equals to $\mathbf{V}_{X,Y}$, the mutual information between $X$ and $Y$ should be zero since $X$ is independent with $Y$. Intuitively, the ``distance'' between $\mathbf{U}_{X,Y}$ and $\mathbf{V}_{X,Y}$ represents the mutual information between them.

\begin{definition}[$f$-mutual information \cite{2016arXiv160501021K}]
The $f$-mutual information between $X$ and $Y$ is defined as $$MI^f(X;Y)=D_f(\mathbf{U}_{X,Y},\mathbf{V}_{X,Y})$$ where $D_f$ is $f$-divergence. $f$-mutual information is always non-negative \cite{2016arXiv160501021K}.
\end{definition}

$f$-mutual information is used in the peer prediction literature since if the information is measured by $f$-mutual information, any ``data processing'' on either of the random variables will decrease the amount of information crossing them. Thus, in peer prediction, if we pay agents according to the $f$-mutual information between her information and her peers' information, agents will be incentivized to report all information to maximize their payments\footnote{In the current paper, we do not directly use the data processing inequality of $f$-mutual information. Thus, we omit the formal introduction here. The interested reader is refer to \citet{2016arXiv160501021K}. }. 

Two examples of $f$-mutual information are Shannon mutual information~\cite{cover2006elements} (Choosing $f$-divergence as KL divergence) and $MI^{tvd}(X;Y):=\sum_{x,y}|\Pr[X=x,Y=y]-\Pr[X=x]\Pr[Y=y]|$ (Choosing $f$-divergence as Total Variation Distance).

We define $K(X=x,Y=y)$ as the ratio between $U_{X,Y}(x,y)$ and $V_{X,Y}(x,y)$, i.e., $$K(X=x,Y=y):=\frac{\Pr[X=x,Y=y]}{\Pr[X=x]\Pr[Y=y]}=\frac{\Pr[Y=y|X=x]}{\Pr[Y=y]}=\frac{\Pr[X=x|Y=y]}{\Pr[X=x]}.$$ $K(X=x,Y=y)$ represents the ``\textbf{pointwise mutual information}(PMI)'' between $X=x$ and $Y=y$. Lemma~\ref{lemma:dualdivergence} directly implies:

\begin{lemma}[Dual version of $f$-mutual information]\label{lemma:duality}
$$ MI^f(X;Y) \geq \sup_{u\in \mathcal{G}} \E_{U_{X,Y}} u- \E_{V_{X,Y}}f^{\star}(u)$$
where $\mathcal{G}$ is a set of functions that maps $\Sigma_X\times \Sigma_Y$ to $\mathbb{R}$. 

The equality holds if and only if $u(x,y)=u^*(x,y)\in \partial{f}(K(X=x,Y=y))$.
\end{lemma}

\setlength{\textfloatsep}{0pt}
\begin{table}\label{table:distinguishers}
\centering
\begin{tabular}{llll} 
    \toprule
    {$f$-divergence} & {$f(t)$} & {$u^*(x,y)=\partial{f}(K(x,y))$} & {$f^{\star}(u^*(x,y)$)} \\ \midrule
    Total Variation Distance & $|t-1|$  & sign($\log K(x,y)$) & sign($\log K(x,y)$) \\ 
    \midrule
    KL divergence & $t\log t$  & $1+\log K(x,y)$ & $K(x,y)$ \\ 
    \midrule
    Reverse KL & $-\log t$  & $-\frac{1}{K(x,y)}$ & $-1+\log K(x,y)$) \\ 
    \midrule
    Pearson $\chi^2$ & $(t-1)^2$  & $2(K(x,y)-1)$ & $(K(x,y))^2-1$ \\ 
    \midrule
    Squared Hellinger & $(\sqrt{t}-1)^2$  & $1-\sqrt {\frac{1}{K(x,y)}}$ & $\sqrt {K(x,y)}-1$\\ 
    \bottomrule
\end{tabular}
\caption{Reference for common $f$-divergences and corresponding pairs of best distinguishers $(u^*(x,y),f^{\star}(u^*(x,y))$ of $f$-mutual information. $K(x,y)=K(X=x,Y=y)$ (PMI).}
\end{table}

\subsection{Proper scoring rules} 
A scoring rule $PS:  \Sigma \times \Delta_{\Sigma} \mapsto \mathbb{R}$ \cite{winkler1969scoring,gneiting2007strictly} takes in a signal $\sigma \in \Sigma$  and a distribution over signals $\mathbf{p} \in \Delta_{\Sigma}$ and outputs a real number.  A scoring rule is \emph{proper} if, whenever the first input is drawn from a distribution $\mathbf{p}$, then $\mathbf{p}$ will maximize the expectation of $PS$ over all possible inputs in $\Delta_{\Sigma}$ to the second coordinate. A scoring rule is called \emph{strictly proper} if this maximum is unique. We will assume throughout that the scoring rules we use are strictly proper. Slightly abusing notation, we can extend a scoring rule to be $PS:  \Delta_{\Sigma} \times \Delta_{\Sigma} \mapsto \mathbb{R}$  by simply taking $PS(\mathbf{p}, \mathbf{q}) = \E_{\sigma \leftarrow \mathbf{p}}(\sigma,  \mathbf{q})$.  We note that this means that any proper scoring rule is linear in the first term. 

%\begin{lemma}[Convexity of $PS$~\cite{winkler1969scoring,gneiting2007strictly}]For any $\mathbf{p},\mathbf{q}\in \Delta_{\Sigma}$, for any $0\leq \lambda\leq 1$, $$ PS(\lambda\mathbf{p}+(1-\lambda)\mathbf{q},\lambda\mathbf{p}+(1-\lambda)\mathbf{q})\leq \lambda PS(\mathbf{p},\mathbf{p})+(1-\lambda)PS(\mathbf{q},\mathbf{q}).  $$\end{lemma}

%\begin{lemma}[Convexity of Proper scoring rules]\cite{gneiting2007strictly}$PS(\mathbf{p},\mathbf{p})$ is a convex function of$\mathbf{p}$.\end{lemma}

\begin{example}[Log Scoring Rule~\cite{winkler1969scoring,gneiting2007strictly}]\label{eg:lsr}
Fix an outcome space $\Sigma$ for a signal $\sigma$.  Let $\mathbf{q} \in \Delta_{\Sigma}$ be a reported distribution.
The Logarithmic Scoring Rule maps a signal and reported distribution to a payoff as follows:
$$LSR(\sigma,\mathbf{q})=\log (\mathbf{q}(\sigma)).$$

Let the signal $\sigma$ be drawn from some random process with distribution $\mathbf{p} \in \Delta_\Sigma$.

Then the expected payoff of the Logarithmic Scoring Rule
$$ \E_{\sigma \leftarrow \mathbf{p}}[LSR(\sigma,\mathbf{q})]=\sum_{\sigma}\mathbf{p}(\sigma)\log \mathbf{q}(\sigma)=LSR(\mathbf{p},\mathbf{q})$$

This value will be maximized if and only if $\mathbf{q}=\mathbf{p}$.

\end{example}

\subsection{Property of the pointwise mutual information}

We will introduce a simple property of the pointwise mutual information that we will use multiple times in the future. In addition to several different formats of the pointwise mutual information (e.g. joint distribution/product of the marginal distributions, posterior/prior), if there exists a latent random variable $Y$ such that random variable $X_A$ and random variable $X_B$ are independent conditioning on $Y$, we can also represent the pointwise mutual information between $X_A$ and $X_B$ by the ``agreement'' between the ``relationship'' between $X_A$ and $Y$, and the ``relationship'' between $X_B$ and $Y$.

\begin{claim}\label{claim:ci}
When random variables $X_A$, $X_B$ are independent conditioning on $Y$, 
\begin{align*}
    K(X_A=x_A,X_B=x_B)
    =&\sum_y {\Pr[Y=y]}K(X_A=x_A,Y=y) K(X_B=x_B,Y=y)\\
    =&\sum_y \Pr[Y=y|X_A=x_A] K(X_B=x_B,Y=y)\\
    =&\sum_y \frac{\Pr[Y=y|X_A=x_A]\Pr[Y=y|X_B=x_B]}{\Pr[Y=y]}.
\end{align*}
\end{claim}

We defer the proof to the appendix.

\section{General Model and Assumptions}\label{sec:model}

Let $X_A,X_B,Y$ be three random variables and we define prior $Q$ as the joint distribution over $X_A,X_B,Y$. We want to forecast the ground truth $Y$ whose realization is a signal in a finite set $\Sigma$. $X_A, X_B$ are two sources of information that are related to $Y$. $X_A$'s realization is a signal in a finite set $\Sigma_A$. $X_B$'s realization is a signal in a finite set $\Sigma_B$. We may have access to both of the realizations of $X_A$ and $X_B$ or only one of them. Thus, we need to learn the relationship between $X_A, X_B$ and $Y$ to forecast $Y$. It's impossible to learn by only accessing the samples of $X_A, X_B$ without additional assumption. We make the following conditional independence assumption:

\begin{assumption}[Conditional independence]\label{assume:coni}
    We assume that conditioning on $Y$, $X_A$, and $X_B$ are independent. 
\end{assumption}

Intuitively, $Y$ can be seen as the ``intersection'' between $X_A$ and $X_B$. To better understand this assumption and its limitations we return to our running example where the variable $Y$ is the success of a start-up.  In this case, if both $X_A$ and $X_B$ contain the sex of the CEO (which we assume is independent of $Y$), then this assumption will not hold.  To make it hold, either $Y$ would need to be redefined to contain the sex of the CEO, or this information would need to be removed from either $X_A$ or $X_B$. For the mechanism design application, if the assumption is violated, for example both agents are sexists and forecast using the sex of the CEO, then it is impossible to avoid paying them for this useless/harmful information.

% \gs{new!} This assumption would be violated, if, for example, both  $X_A$ and $X_B$ contain the sex of the CEO, and predictors $P_A$ and $P_B$ are systematically biased due to the sex of the CEO.  

% \gs{new!} This assumption states that $X_A$ and $X_B$ have no common information (which $P_A$ and $P_B$ exploit) that is independent of $Y$. This would be violated, if, for example, both  $X_A$ and $X_B$ contain the sex of the CEO, then the ``common ground" truth $Y$ (the 

% assumption states that $X_A$ and $X_B$ have no common information (which $P_A$ and $P_B$ exploit) that is independent of $Y$.  This would be violated, if, for example, both  $X_A$ and $X_B$ contain the sex of the CEO, and predictors $P_A$ and $P_B$ are systematically biased due to the sex of the CEO.  

\subsection{Well-defined and stable prior} 

%%%intuition!!!!!!motivation!!!!!!!!
%%%%% hypotheses / deep learning example / low vc dimension

%we reduce the problem to the optimization problem and as usual related literature, we do not investigate the computation complexity and data requirement

%%%computation 
%%%sample complexity and expected mutual information

We call $Z$ a \emph{solution} if conditioning on $Z$, $X_A$, and $X_B$ are independent. $Y$ is a solution. However, there are a lot of solutions. For example, conditioning on $X_A$ or $X_B$, $X_A$ and $X_B$ are independent, which means $X_A$ and $X_B$ are both solutions. Thus, we have an additional restriction on the prior: well-defined prior and stable prior. 

We will need restrictions on the prior when we analyze the strictness of our learning algorithm/mechanism. Readers can skip this section without losing the core idea of our results.

To infer the relationship between $Y$ and $X_A,X_B$ with only samples of $X_A,X_B$, we cannot do better than to just solve the system of equations (\ref{soe}), given the joint distribution over $X_A,X_B$: $Q$. Our goal is to obtain the Bayesian posterior predictor. Thus, we list a system that the Bayesian posterior predictor satisfies. The system below equations involve variables $\{\mathbf{a}^{x_A},\mathbf{b}^{x_B}\in\Delta_{\Sigma}\}_{x_A\in \Sigma_A,x_B\in \Sigma_B}$, and $\mathbf{r}\in\Delta_{\Sigma}$.   We insist $a^{x_A}_y=\Pr[Y=y|X_A=x_A]$, $b^{x_B}_y=\Pr[Y=y|X_B=x_B]$ and $r_y = \Pr[Y=y]$ is a solution and we call it the \emph{desired} solution.

\begin{align}\label{soe}
\mathcal{S}(\{\mathbf{a}^{x_A},&\mathbf{b}^{x_B}\}_{x_A\in \Sigma_A,x_B\in \Sigma_B},\mathbf{r})\\ \nonumber
:=&\bigg\{\sum_{y\in \Sigma} \frac{a^{x_A}_y b^{x_B}_y}{r_y}-K(X_A=x_A,X_B=x_B)\bigg\}_{x_A\in \Sigma_A,x_B\in \Sigma_B}=0
 \end{align}
 
Claim~\ref{claim:ci} shows the above system has the desired solution. 

Note that any permutation of a solution is still a valid solution\footnote{We may be able to distinguish a solution with its permuted version if we have some side information (e.g. the prior of $Y$/a few $(x_A,x_B,y)$ samples).}. Since we cannot do better than to solve the above system, if the above system only has one ``unique'' solution, in the sense that any two solutions are permuted version of each other, we call the prior $Q$ a well-defined prior. Formally,

\begin{definition}[Well-defined]
    A prior $Q$ is well-defined if for any two solutions $\{\mathbf{a}^{x_A},\mathbf{b}^{x_B}\}_{x_A\in \Sigma_A,x_B\in \Sigma_B}$, $\mathbf{r}$ and $\{\mathbf{c}^{x_A},\mathbf{d}^{x_B}\}_{x_A\in \Sigma_A,x_B\in \Sigma_B}$, $\mathbf{r}'$ of the system of equations (\ref{soe}), there exists a permutation $\pi: \Sigma\mapsto\Sigma$ such that $\mathbf{r}=\pi \mathbf{r}' $ for any $x_A,x_B$, $\mathbf{a}^{x_A}=\pi \mathbf{c}^{x_A} $, $\mathbf{b}^{x_B}=\pi \mathbf{d}^{x_B} $.
\end{definition}

The well-defined prior exist since intuitively, if $|\Sigma_A|$ and $|\Sigma_B|$ are high and $|\Sigma|$ is low, it is likely $Y$ is the ``unique intersection'' since the number of constraints of the system will be much greater than the number of variables.

We say a prior is stable if fixing part of the desired solution of the system (\ref{soe}), in order to make it still a solution of the system, other parts of the desired solution should also be fixed.

\begin{definition}[Stable]
    A prior $Q$ is stable if fixing 
    $a^{x_A}_y=\Pr[Y=y|X_A=x_A]$ and $r_y = \Pr[Y=y]$, the system (\ref{soe}) $\mathcal{S}(\{\mathbf{a}^{x_A},\mathbf{b}^{x_B}\}_{x_A\in \Sigma_A,x_B\in \Sigma_B},\mathbf{r})=0$ has unique solution $\mathbf{b}^{x_A}$ such that $b^{x_B}_y=\Pr[Y=y|X_B=x_B]$; and fixing 
    $b^{x_B}_y=\Pr[Y=y|X_B=x_B]$ and $r_y = \Pr[Y=y]$, the system (\ref{soe}) $\mathcal{S}(\{\mathbf{a}^{x_A},\mathbf{b}^{x_B}\}_{x_A\in \Sigma_A,x_B\in \Sigma_B},\mathbf{r})=0$ has unique solution $\mathbf{a}^{x_A}$ such that $a^{x_A}_y=\Pr[Y=y|X_A=x_A]$.
\end{definition}

We require stable priors when we design \emph{strictly} truthful mechanisms.

\subsection{Predictors}
This section gives the definition of predictors. We have two sets of samples $S_A:=\{x_A^{\ell}\}_{\ell\in {\mathcal{L}_A}}$ and $S_B:=\{x_B^{\ell}\}_{\ell\in {\mathcal{L}_B}}$ which are i.i.d samples of $X_A$ and $X_B$ respectively. For $\ell\in \mathcal{L}_A\cap\mathcal{L}_B$, $(x_A^{\ell},x_B^{\ell})$s are i.i.d samples of the joint random variable $(X_A,X_B)$.

A predictor $P_A:\Sigma_A\mapsto\Delta_{\Sigma}$ for $X_A$ maps $x_A\in\Sigma$ to a forecast $P_A(x_A)$ for ground truth $Y$. We similarly define the predictors for $X_B$. We define \emph{the Bayesian posterior predictor} as the predictor that maps any input information $X=x$ to its Bayesian posterior forecast for $Y=y$, i.e., $Pr(Y=y|X=x)$. %We call the predictor that maps $x_A$ to its posterior $(\Pr[Y=y|X_A=x_A])_y$ \emph{Bayesian posterior predictor}. %Note that it is the optimal predictor. 

With the conditional independence assumption, we have 
\begin{align*}
    \Pr[Y|X_A,X_B]=&\frac{\Pr[Y,X_A,X_B]}{\Pr[X_A,X_B]}\\ \tag{conditional independence}
    =& \frac{\Pr[Y]\Pr[X_A|Y]\Pr[X_B|Y]}{\Pr[X_A,X_B]}\\ \tag{$K(X_A,X_B)$ is the pointwise mutual information.}
    =& \frac{\Pr[Y|X_A]\Pr[Y|X_B]}{K(X_A,X_B)\Pr[Y]}\\
\end{align*}

When we have access to both the sources where $X_A=x_A$ and $X_B=x_B$, given the prior of the ground truth $Y$, we can construct an aggregated forecast for $Y=y$ using $P_A,P_B$: 

$$\frac{P_A(x_A)P_B(x_B)}{\Pr[Y=y]}\cdot\text{normalization}$$

%explain bayesian posteriors

In this case, if both $P_A$ and $P_B$ are the Bayesian posterior predictor, the aggregated forecast is the Bayesian posterior predictor as well. Thus, it's sufficient to only train $P_A$ and $P_B$. In the rest sections, we will show how to train $P_A$ and $P_B$ (Section~\ref{sec:commontruth}), given the two sets of samples $S_A$ and $S_B$, as well as how to incentivize high quality predictors from the crowds (Section~\ref{sec:forecastelicitation}).

\section{Co-training: finding the common ground truth}\label{sec:commontruth}

We have a set of candidates $\mathcal{H}_A$ for the predictor for $X_A$ and a set of candidates $\mathcal{H}_B$ for the predictor for $X_B$. We sometimes call each predictor candidate \emph{a hypothesis}. Given the two sets of samples $S_A=\{x_A^{\ell}\}_{\ell\in {\mathcal{L}_A}}$ and $S_B=\{x_B^{\ell}\}_{\ell\in {\mathcal{L}_B}}$, our goal is to figure out the best hypothesis in $\mathcal{H}_A$ and the best hypothesis in $\mathcal{H}_B$ simultaneously. Thus, we need to design proper ``loss function'' such that the best hypotheses minimize the loss. In fact, we will show how to design a proper ``reward function'' such that the best hypotheses maximize the reward.

\subsection{$f$-mutual information gain}\label{sec:fgain}

%In this section, we give a formal definition of $f$-mutual information gain. Recall that Figure~\ref{fig:fgain} provides an implementation example.

\paragraph{$f$-mutual information gain $MIG^f(R)$ (Figure~\ref{fig:fgain})}

\begin{description} 
\item[Hypothesis] We are given $\mathcal{H}_A=\{h_A:\Sigma_A\mapsto \Delta_{\Sigma}\}$, $\mathcal{H}_B=\{h_B:\Sigma_B\mapsto \Delta_{\Sigma}\}$: the set of hypotheses/predictor candidates for $X_A$ and  $X_B$, respectively.

\item[Gain] Given reward function $R:\Delta_{\Sigma}\times\Delta_{\Sigma}\mapsto \mathbb{R}$, \\ for each $\ell\in \mathcal{L}_A\cap \mathcal{L}_B$, reward ``the amount of agreement'' between the two predictor candidates' predictions for task $\ell$, i.e., 

$$R(h_A(x_A^{\ell}),h_B(x_B^{\ell}));$$ for each distinct pair $(\ell_A,\ell_B), \ell_A\in \mathcal{L}_A,\ell_B\in \mathcal{L}_B,\ell_A\neq \ell_B$, punish both predictor candidates ``the amount of agreement'' between their predictions for a pair of distinct tasks $(\ell_A,\ell_B)$, i.e., 
$$f^{\star}(R(h_A(x_A^{\ell_A}),h_B(x_B^{\ell_B})).$$

The $f$-mutual information gain $MIG^f(R)$ that is corresponding to the reward function $R$ is 
\begin{align*}
    MIG^f(R(h_A,h_B))_{|S_A,S_B}=&\frac{1}{|\mathcal{L}_A\cap \mathcal{L}_B|}\sum_{\ell \in \mathcal{L}_A\cap \mathcal{L}_B} R(h_A(x_A^{\ell}),h_B(x_B^{\ell}))\\
    -\frac{1}{|\mathcal{L}_A||\mathcal{L}_B|-|\mathcal{L}_A\cap \mathcal{L}_B|^2}&\sum_{\ell_A\in \mathcal{L}_A,\ell_B\in \mathcal{L}_B,\ell_A\neq \ell_B}f^{\star}(R(h_A(x_A^{\ell_A}),h_B(x_B^{\ell_B}))) 
\end{align*}
\end{description}

\begin{lemma}\label{lem:pplearn}
The expected total $f$-mutual information gain is maximized over all possible $R$, $h_A$, and $h_B$ if and only if for any $(x_A,x_B)\in\Sigma_A\times\Sigma_B$, $$R(h_A(x_A),h_B(x_B))\in \partial{f}(K(x_A,x_B)).$$ The maximum is $MI^f(X_A;X_B).$
\end{lemma}

\begin{proof}
$(x_A^{\ell},x_B^{\ell})_{\ell}$ are i.i.d. realizations of $(X_A,X_B)$. Therefore, the expected $f$-mutual information gain is $\E_{U_{X_A,X_B}} R - \E_{V_{X_A,X_B}} f^{\star}(R).$ The results follow from Lemma~\ref{lemma:duality}.
\end{proof}

Although any reward function corresponds to an $f$-mutual information gain function, we need to properly design the reward function $R$ such that, fixing $R$, there exist hypotheses to maximize the corresponding $f$-mutual information gain $MIG^f(R)$ to the $f$-mutual information between the two sources. We will use the intuition from Lemma~\ref{lem:pplearn} to design such reward functions $R$ in the next section.

%We call such reward function \emph{permissible} reward function and give a formal definition in the below paragraph. 

%\begin{definition}[Permissible reward function]Given convex function $f:\mathbb{R}\mapsto\mathbb{R}$, reward function $R:\Delta_{\Sigma}\times\Delta_{\Sigma}\mapsto \mathbb{R}$, $R$ is \emph{permissible} if and only if there exists hypotheses $h_A^*:\Sigma_A\mapsto \Delta_{\Sigma}$, $h_B^*:\Sigma_B\mapsto \Delta_{\Sigma}$ such that $\forall \ell_1,\ell_2$, $$R(h_A^*(x_A^{\ell_1}),h_B^*(x_B^{\ell_2}))=f'(K(x_A^{\ell_1},x_B^{\ell_2})).$$\end{definition}

\subsection{Maximizing the $f$-mutual information gain}

%\begin{definition}[$R^f_{AP}$]Given the prior over $Y$, for $\Delta_{\Sigma}=\Delta_{\Sigma}=\Delta_{\Sigma}$, we define loss function $R^f_{AP}$ as a function that maps $(\mathbf{p}_1,\mathbf{p}_2)$ to  $$R^f_{AP}(\mathbf{p}_1,\mathbf{p}_2):=-R(\mathbf{p}_1,\mathbf{p}_2)=-f'(\sum_{y}\frac{\mathbf{p}_1(y) \mathbf{p}_2(y)}{\Pr[Y=y]}).$$\end{definition}

%\begin{definition}[$R^f$]Given the prior over $Y$, for $\Delta_{\Sigma}=\Delta_{\Sigma}=\Delta_{\Sigma}$, we define loss function $R^f$ as a function that maps $(\mathbf{p}_1,\mathbf{p}_2)$ to  $$R^f(\mathbf{p}_1,\mathbf{p}_2):=-f'(\sum_{y}{\mathbf{p}_1(y) \mathbf{p}_2(y)}{\Pr[Y=y]}).$$\end{definition}

In this section, we will construct a special reward function $R^f$ and then show that the maximizers of the corresponding $f$-mutual information gain $MIG^f(R^f)$ are the Bayesian posterior predictors.

\begin{definition}[$R^f$]\label{def:agree} We define reward function $R^f$ as a function that maps the two hypotheses' outputs $\mathbf{p}_1,\mathbf{p}_2\in \Delta_{\Sigma}$ and the vector $\mathbf{p}\in \Delta_{\Sigma}$ to  $$R^f(\mathbf{p}_1,\mathbf{p}_2,\mathbf{p}):=g\bigg(\sum_{y}{\frac{\mathbf{p}_1(y) \mathbf{p}_2(y)}{\mathbf{p}(y)} }\bigg)$$ where $g(t)\in \partial{f}(t),\forall t$. When $f$ is differentiable, $$R^f(\mathbf{p}_1,\mathbf{p}_2,\mathbf{p}):=f'\bigg(\sum_{y}{\frac{\mathbf{p}_1(y) \mathbf{p}_2(y)}{\mathbf{p}(y)} }\bigg).$$\end{definition}

With this definition of the reward function, fixing $\mathbf{p}\in \Delta_{\Sigma}$ which can be seen as the prior over $Y$, the ``amount of agreement'' between two predictions $\mathbf{p}_1,\mathbf{p}_2$ are an increasing function $g$ of 

$$\sum_{y}{\frac{\mathbf{p}_1(y) \mathbf{p}_2(y)}{\mathbf{p}(y)} }, $$ which is intuitive and reasonable. The increasing function $g$ is the derivative of the convex function $f$. By carefully choosing convex function $f$, we can use any increasing function $g$ here.

\begin{example}
Here we present some examples of the $f$-mutual information gain $MIG^f(R^f)$ with reward function $R^f$, associated with different $f$-divergences. We use Table 1 as reference for $\partial{f}(\cdot)$ and $f^{\star}(\partial{f}(\cdot))$. 

\vspace{5pt}

Total variation distance:
\begin{align*}
    &\frac{1}{|\mathcal{L}_A\cap \mathcal{L}_B|}\sum_{\ell \in \mathcal{L}_A\cap \mathcal{L}_B}
    sign\bigg(log[\sum_{y}{\frac{h_A(x_A^{\ell})(y) h_B(x_B^{\ell})(y)}{\mathbf{p}(y)} }]\bigg)\\
    &-\frac{1}{|\mathcal{L}_A||\mathcal{L}_B|-|\mathcal{L}_A\cap \mathcal{L}_B|^2}\sum_{\ell_A\in \mathcal{L}_A,\ell_B\in \mathcal{L}_B,\ell_A\neq \ell_B} sign\bigg(log[\sum_{y}{\frac{h_A(x_A^{\ell_A})(y) h_B(x_B^{\ell_B})(y)}{\mathbf{p}(y)} }]\bigg)
\end{align*}

KL divergence:

\begin{align*}
    &\frac{1}{|\mathcal{L}_A\cap \mathcal{L}_B|}\sum_{\ell \in \mathcal{L}_A\cap \mathcal{L}_B}
    \bigg(1+log[\sum_{y}{\frac{h_A(x_A^{\ell})(y) h_B(x_B^{\ell})(y)}{\mathbf{p}(y)} }]\bigg)\\
    &-\frac{1}{|\mathcal{L}_A||\mathcal{L}_B|-|\mathcal{L}_A\cap \mathcal{L}_B|^2}\sum_{\ell_A\in \mathcal{L}_A,\ell_B\in \mathcal{L}_B,\ell_A\neq \ell_B} \bigg(\sum_{y}{\frac{h_A(x_A^{\ell_A})(y) h_B(x_B^{\ell_B})(y)}{\mathbf{p}(y)} }\bigg)
\end{align*}

Pearson: 
\begin{align*}
    &\frac{1}{|\mathcal{L}_A\cap \mathcal{L}_B|}\sum_{\ell \in \mathcal{L}_A\cap \mathcal{L}_B}
    2*\bigg(\sum_{y}{\frac{h_A(x_A^{\ell})(y) h_B(x_B^{\ell})(y)}{\mathbf{p}(y)} }-1\bigg)\\
    &-\frac{1}{|\mathcal{L}_A||\mathcal{L}_B|-|\mathcal{L}_A\cap \mathcal{L}_B|^2}\sum_{\ell_A\in \mathcal{L}_A,\ell_B\in \mathcal{L}_B,\ell_A\neq \ell_B} \bigg((\sum_{y}{\frac{h_A(x_A^{\ell_A})(y) h_B(x_B^{\ell_B})(y)}{\mathbf{p}(y)} })^2-1\bigg)
\end{align*}
\end{example}

\begin{theorem}\label{thm:ppl}
With the conditional independent assumption on $X_A,X_B,Y$, given the samples $S_A,S_B$, given a convex function $f$, we define the optimization goal as the expected $f$-mutual information gain with reward function $R^f$, i.e., 
\begin{align*}
    MIG^f(h_A,h_B,\mathbf{p}):=\E_{X_A,X_B}MIG^f(R^f(h_A,h_B,\mathbf{p}))_{|S_A,S_B}\\
    %&-||\E_{X_A}h_A(X_A)-(1,1,...,1)||_2 %-||\E_{X_B}h_B(X_B)-(1,1,...,1)||_2\\
\end{align*} and optimize over all possible hypotheses $h_A:\Sigma_A\mapsto \Delta_{\Sigma}$, $h_B:\Sigma_B\mapsto \Delta_{\Sigma}$ and distribution vectors $\mathbf{p}\in \Delta_{\Sigma}$. We have

\begin{description}
    \item [Solution$\rightarrow$Maximizer:] any solution $Z$ corresponds to a maximizer of $MIG^f(h_A,h_B,\mathbf{p})$\footnote{Given the prior over $Y$, we can fix $\mathbf{p}$ as the prior over $Y$. Without knowing the prior over $Y$, $\mathbf{p}$ becomes a variable of the optimization goal and helps us learn the prior over $Y$. }: for any solution $Z$, $$h_A^*(x_A):=(\Pr[Z=y|X_A=x_A])_y\qquad h_B^*(x_B):=(\Pr[Z=y|X_B=x_B])_y\footnote{Recall that we use $(\phi(y))_{y\in[N]}$ to represent the vector $(\phi(1),\phi(2),...,\phi(N))\in \mathbb{R}^N$. }$$ and the prior over $Z$, $\Pr[Z=y]_y$, is the maximizer of $MIG^f(h_A,h_B,\mathbf{p})$ and the maximum is $MI^f(X_A;X_B)$;

\item [Maximizer$\rightarrow$(Permuted) Ground truth] when the prior is well-defined, $f$ is differentiable, and $f'$ is invertible, any maximizer of $MIG^f(h_A,h_B,\mathbf{p})$ corresponds to the (possibly permuted) ground truth $Y$: for any maximizer $(h_A^*(\cdot),h_B^*(\cdot),\mathbf{p}^*)$ of $MIG^f(h_A,h_B,\mathbf{p})$, there exists a permutation $\pi$ such that $$h_A^*(x_A):=(\Pr[\pi(Y)=y|X_A=x_A])_y\qquad h_B^*(x_B):=(\Pr[\pi(Y)=y|X_B=x_B])_y$$ and $\mathbf{p}^*=\Pr[\pi(Y)=y]_y$. %Moreover, for the corresponded solution $Y$, the posterior probability  and likelihood of $Y=y$ conditioning on the two peer instances $x_A,x_B$ are proportional to $$\mathbf{p}^*(y)h_A^*(x_A)(y) h_B^*(x_B)(y)\qquad h_A^*(x_A)(y) h_B^*(x_B)(y)$$ correspondingly.

\end{description}

%Moreover, when the prior is well-defined, then every maximizer of $MIG^f(h_A,h_B,\mathbf{p})$ corresponds to (possibly permuted version) the ground truth $Y$. 

\end{theorem}

The above theorem neither investigates computation complexity (which may be affected by the choice of $f$), data requirements, nor the choice of the hypothesis class for practical implementation (see Section~\ref{sec:discussions} for more discussion).

\begin{proof}[Proof for Theorem~\ref{thm:ppl}]

Lemma~\ref{lem:pplearn} shows that the expected $f$-mutual information gain is maximized if and only if for any $(x_A,x_B)$, $$R^f(h_A^*(x_A),h_B^*(x_B),\mathbf{p}^*)\in \partial{f}(K(x_A,x_B)).$$ 

(1)\emph{ Solution$\rightarrow$Maximizer:} For any solution $Z$, we can construct $$h_A^*(x_A):=(\Pr[Z=y|X_A=x_A])_y\qquad h_B^*(x_B):=(\Pr[Z=y|X_B=x_B])_y$$ and $\mathbf{p}^*=\Pr[Z=y]_y$. Then 

\begin{align*}
    R^f(h_A^*(x_A),h_B^*(x_B),\mathbf{p}^*)&\in \partial{f}\bigg(\sum_{y}{\frac{\Pr[Z=y|X_A=x_A]\Pr[Z=y|X_B=x_B]}{\Pr[Z=y]}}\bigg)\\ \tag{Claim~\ref{claim:ci}}
    &= \partial{f}(K(x_A,x_B)
\end{align*}

Thus, based on Lemma~\ref{lem:pplearn}, any solution $Z$ corresponds to a maximizer of the optimization goal. 

(2)\emph{Maximizer$\rightarrow$(Permuted) Ground truth:} For any maximizer $(h_A^*(\cdot),h_B^*(\cdot),\mathbf{p}^*)$ of the optimization goal, when $f$ is differentiable, Lemma~\ref{lem:pplearn} shows that  $$R^f(h_A^*(x_A),h_B^*(x_B),\mathbf{p}^*)=f'(K(x_A,x_B)).$$ When $f'$ is invertible, we have 
\begin{align*}
    \sum_{y}{\frac{h_A^*(x_A)(y) h_B^*(x_B)(y)}{\mathbf{p}^*(y)}}=K(x_A,x_B)
\end{align*} for all $x_A,x_B$.

Thus, $\{(h_A^*(x_A),h_B^*(x_B),\mathbf{p}^*)\}_{x_A,x_B}$ is actually the solution of the system (\ref{soe}). When the prior is well-defined, there exists a permutation $\pi$ such that $$h_A^*(x_A):=(\Pr[\pi(Y)=y|X_A=x_A])_y\qquad h_B^*(x_B):=(\Pr[\pi(Y)=y|X_B=x_B])_y$$ and $\mathbf{p}^*=\Pr[\pi(Y)=y]_y$ where $Y$ is the ground truth.

\end{proof}

\section{Forecast elicitation without verification}\label{sec:forecastelicitation}
This section considers the setting where the forecasts are provided by the crowds and we want to incentivize high quality forecast by providing an instant reward without instant access to the ground truth. 

There is a forecasting task. Alice and Bob have private information $X_A,X_B=x_A\in\Sigma_A,x_B\in\Sigma_B$ correspondingly and are asked to forecast the ground truth $Y=y$. We denote $(\Pr[Y=y|X_A=x_A])_y$, $(\Pr[Y=y|X_B=x_B])_y$ by $\mathbf{p}_{x_A}$, $\mathbf{p}_{x_B}$ correspondingly. Alice and Bob are asked to report their Bayesian forecast $\mathbf{p}_{x_A}$, $\mathbf{p}_{x_B}$. We denote their actual reports by $\hat{\mathbf{p}}_{x_A}$ and $\hat{\mathbf{p}}_{x_B}$. Without access to the realization of $Y$, we want to incentivize both Alice and Bob play \emph{truth-telling} strategies, i.e., honestly reporting their forecast $\mathbf{p}_{x_A}$, $\mathbf{p}_{x_B}$ for $Y$.

We define the \emph{strategy} of Alice as a mapping $s_A$ from $x_A$ (private signal) to a probability distribution over the space of all possible forecast for random variable $Y$. Analogously, we define Bob's strategy $s_B$. Note that essentially each (possibly mixed) strategy $s_A$ can be seen as a (possibly random) predictor $P_A$ where $P_A(x_A)$ is a random forecast drawn from distribution $s_A(x_A)$. In particular, the truthful strategy corresponds to the Bayesian posterior predictor.

We say agents play a \emph{permutation strategy profile} if there exists permutation $\pi:\Sigma\mapsto \Sigma$ such that each agent always reports $\pi \mathbf{p}$ given her truthful report is $\mathbf{p}$. 

Note that without any side information about $Y$, we cannot distinguish the scenario where agents are honest and the scenario where agents play a permutation strategy profile. Thus, it is too much to ask truth-telling to be strictly better than any other strategy profile. The focal property defined in the following paragraph is the optimal property we can obtain.

\paragraph{Mechanism Design Goals}
\begin{description}
\item[(Strictly) Truthful] Mechanism $\mathcal{M}$ is (strictly) truthful if truth-telling is a (strict) equilibrium. 
\item[Focal] Mechanism $\mathcal{M}$ is focal if it is strictly truthful and each agent's expected payment is maximized if agents tell the truth; moreover, when agents play a non-permutation strategy profile, each agent's expected payment is \emph{strictly} less.
\end{description}

We consider two settings:
\begin{description}
\item[Multi-task] Each agent is assigned several independent a priori similar forecasting tasks in a random order and is asked to report her forecast for each task. 
\item[Single-task] All agents are asked to report their forecast for the same single task. 
\end{description}

In the single-task setting, it's impossible to design focal mechanisms since agents can collaborate to pick an arbitrary $y^*\in\Sigma$ and pretend that they know $Y=y^*$. However, we will show we can design strictly truthful mechanism in the single-task setting. In the multi-task setting, since agents may be assigned different tasks and the tasks show in random order, they cannot collaborate to pick an arbitrary $y^*\in\Sigma$ for each task. In fact, we will show if the number of tasks is greater or equal to 2, we can design a family of focal mechanisms.

Achieving the focal goal in the multi-task setting is very similar to what we did in finding the common ground truth. Note that in the forecast elicitation problem, incentivizing a truthful strategy is equivalent to incentivizing the Bayesian posterior predictor. Thus, we can directly use the $f$-mutual information gain as the reward in the multi-task setting. Achieving the strictly truthful goal in the single-task setting is more tricky and we will return to it later.

\subsection{Multi-task: focal forecast elicitation without verification}\label{sec:multi}

We assume Alice is assigned tasks set $\mathcal{L}_A$ and Bob is assigned tasks set $\mathcal{L}_B$. For each task $\ell$, Alice's private information is $x_A^{\ell}$ and Bob's private information is $x_B^{\ell}$. The ground truth of this task is $y^{\ell}$.

\paragraph{Multi-task common ground mechanism $MCG(f)$} Given the prior distribution over $Y$, a convex and \emph{differentiable} function $f$ whose convex conjugate is $f^{\star}$, 
\begin{description} 
\item[Report] for each task $\ell\in\mathcal{L}_A$, Alice is asked to report $\mathbf{p}_{{x_A}^{\ell}}:=(\Pr[Y=y|x_A^{\ell}])_y$; for each task $\ell\in\mathcal{L}_B$, Bob is asked to report $\mathbf{p}_{{x_B}^{\ell}}:=(\Pr[Y=y|x_B^{\ell}])_y$. We denote their actual reports by $\hat{\mathbf{p}}_{{x_A}^{\ell}}^{\ell}$ and $\hat{\mathbf{p}}_{{x_B}^{\ell}}^{\ell}$. %We denote the tasks Alice chooses to report by $\mathcal{L}_A$ and the tasks Bob choose to report by $\mathcal{L}_B$.
\item[Payment] For each $\ell\in \mathcal{L}_A\cap \mathcal{L}_B$, reward both Alice and Bob ``the amount of agreement'' between their forecast in task $\ell$, i.e., 

$$R(\hat{\mathbf{p}}_{{x_A}^{\ell}}^{\ell},\hat{\mathbf{p}}_{{x_B}^{\ell}}^{\ell});$$ for each pair of distinct tasks $(\ell_A,\ell_B), \ell_A\in \mathcal{L}_A,\ell_B\in \mathcal{L}_B,\ell_A\neq \ell_B$, punish both Alice and Bob ``the amount of agreement'' between their forecast in distinct tasks $(\ell_A,\ell_B)$, i.e., 
$$f^{\star}(R(\hat{\mathbf{p}}_{{x_A}^{\ell_A}}^{\ell_A},\hat{\mathbf{p}}_{{x_B}^{\ell_B}}^{\ell_B}).$$

In total, both Alice and Bob are paid
\begin{align*}
    &\frac{1}{|\mathcal{L}_A\cap \mathcal{L}_B|}\sum_{\ell \in \mathcal{L}_A\cap \mathcal{L}_B} R(\hat{\mathbf{p}}_{{x_A}^{\ell}}^{\ell},\hat{\mathbf{p}}_{{x_B}^{\ell}}^{\ell})\\
    &-\frac{1}{|\mathcal{L}_A||\mathcal{L}_B|-|\mathcal{L}_A\cap \mathcal{L}_B|^2}\sum_{\ell_A\in \mathcal{L}_A,\ell_B\in \mathcal{L}_B,\ell_A\neq \ell_B}f^{\star}(R(\hat{\mathbf{p}}_{{x_A}^{\ell_A}}^{\ell_A},\hat{\mathbf{p}}_{{x_B}^{\ell_B}}^{\ell_B})
\end{align*}

where $$R(\mathbf{p}_1,\mathbf{p}_2):=f'(\sum_{y}\frac{\mathbf{p}_1(y) \mathbf{p}_2(y)}{\Pr[Y=y]}).$$
\end{description}

We do not want agents to collaborate with each other based on the index of the task or other information in addition to the private information. Thus, we make the following assumption to guarantee the index of the task is meaningless for all agents. 

\begin{assumption}[A priori similar and random order]
    For each task $\ell$, fresh i.i.d. realizations of $(X_A,X_B,Y)=(x_A^{\ell},x_B^{\ell},y^{\ell})$ are generated. All tasks appear in a random order, independently drawn for each agent.
\end{assumption}

\begin{theorem}\label{thm:focal}
With the conditional independence assumption, and a priori similar and random order assumption, when the prior $Q$ is stable and well-defined, given the prior distribution over the $Y$, given a differential convex function $f$ whose derivative $f'$ is invertible, if $\max\{|\mathcal{L}_A|,|\mathcal{L}_B|\}\geq 2$, then $MCG(f)$ is focal. 

When both Alice and Bob are honest, each of them's expected payment in $MCG(f)$ is
$$ MI^f(X_A;X_B). $$
\end{theorem}

The non-negativity of $MI^f$ implies that agents are willing to participate in the mechanism. Like Theorem~\ref{thm:ppl}, in order to show Theorem~\ref{thm:focal}, we need to first introduce a lemma which is very similar to Lemma~\ref{lem:pplearn}.

\begin{lemma}\label{lem:focal}
With the conditional independence assumption, the expected total payment is maximized over Alice and Bob's strategies if and only if $\forall \ell_1 \in \mathcal{L}_A, \ell_2 \in \mathcal{L}_B$, for any $(x_A^{\ell_1},x_B^{\ell_2})\in\Sigma_A\times\Sigma_B$, $$R(\hat{\mathbf{p}}_{{x_A}^{\ell_1}}^{\ell_1},\hat{\mathbf{p}}_{{x_B}^{\ell_2}}^{\ell_2})=f'(K(x_A^{\ell_1},x_B^{\ell_2})).$$ The maximum is $MI^f(X_A;X_B).$
\end{lemma}

The proofs of Lemma~\ref{lem:focal} and Theorem~\ref{thm:focal} are very similar with Lemma~\ref{lem:pplearn} and Theorem~\ref{thm:ppl}. We defer the formal proofs to the appendix.

\subsection{Single-task: strictly truthful forecast elicitation without verification}\label{sec:single}
This section introduces the strictly truthful mechanism in the single-task setting. If we know the realization $y$ of $Y$, we can simply apply a proper scoring rule and pay Alice and Bob $PS(y,\hat{\mathbf{p}}_{x_A})$ and $PS(y,\hat{\mathbf{p}}_{x_B})$ respectively. Then according to the property of the proper scoring rule, Alice and Bob will honestly report their truthful forecast to maximize their expected payment. However, we do not know the realization of $Y$. In the information elicitation without verification setting where Alice and Bob are required to report their information, \citet{MRZ05} propose the ``peer prediction'' idea, that is, pays Alice the accuracy of the forecast that predicts Bob's information conditioning Alice's information, i.e., $$PS\big(\hat{x}_B,(\Pr[{X}_B=x_B|\hat{x}_A])_y\big)$$ where $\hat{x}_A$ and $\hat{x}_B$ are Alice and Bob's reported information. We note the peer prediction mechanism in \citet{MRZ05} is truthful. With a similar ``peer prediction'' idea, we propose a strictly truthful mechanism in forecast elicitation.

\paragraph{Common ground mechanism} Given the prior distribution over $Y$, 
\begin{description} 
\item[Report] Alice and Bob are required to report $\mathbf{p}_{x_A}$, $\mathbf{p}_{x_B}$. We denote their actual reports by $\hat{\mathbf{p}}_{x_A}$ and $\hat{\mathbf{p}}_{x_B}$. 
\item[Payment] Both Alice and Bob are paid 
$$ \log \sum_{y}\frac{\hat{\mathbf{p}}_{x_A}(y) \hat{\mathbf{p}}_{x_B}(y)}{\Pr[Y=y]} .$$
\end{description}

\begin{theorem}
With the conditional independence assumption (and when the prior is stable), given the prior distribution over the $Y$, the common ground mechanism is (strictly) truthful;

moreover, when both Alice and Bob are honest, each of them's expected payment in the common ground mechanism is the Shannon mutual information between their private information $$I(X_A;X_B)=MI^{KL}(X_A;X_B).$$
\end{theorem}

%Part of the proof requires $\log \frac{a}{b}=\log a-\log b$. Thus we cannot extend the result to the general proper scoring rule.

The non-negativity of the Shannon mutual information implies that agents are willing to participate in the mechanism. The (strictly) truthful property of the common ground mechanism is proved by the fact that log scoring rule $LSR$ is strictly proper. %We defer the proof to the appendix.

\begin{proof}
When both Alice and Bob are honest, their payment is $\log K(x_A,x_B)$ according to Claim~\ref{claim:ci}. Their expected payment will be 
\begin{align*}
   \sum_{x_A,x_B}\Pr[x_A,x_B] \log K(x_A,x_B) = \sum_{x_A,x_B}\Pr[x_A,x_B] \log \frac{\Pr[x_A,x_B]}{\Pr[x_A]\Pr[x_B]}=MI^{KL}(X_A;X_B)
\end{align*}

Given that Bob honestly reports $\hat{\mathbf{p}}_{x_B}={\mathbf{p}}_{x_B}$, we would like to show that the expected payment of Alice is less than $MI^{KL}(X_A; X_B)$ regardless of the strategy Alice plays. The expected payment of Alice is 
\begin{align*}
     &\sum_{x_A,x_B} \Pr[X_A=x_A,X_B=x_B]\log \sum_{y}\frac{\hat{\mathbf{p}}_{x_A}(y) {\mathbf{p}}_{x_B}(y)}{\Pr[Y=y]}\\
     =&\sum_{x_A,x_B} \Pr[X_A=x_A,X_B=x_B]\log \sum_{y}\frac{\hat{\mathbf{p}}_{x_A}(y) {\mathbf{p}}_{x_B}(y)}{\Pr[Y=y]} \Pr[X_B=x_B]\\
     &-\sum_{x_A,x_B} \Pr[X_A=x_A,X_B=x_B]\log \Pr[X_B=x_B]\\ \tag{$C$ is a constant that does not depend on Alice's strategy}
     =&\sum_{x_A,x_B} \Pr[X_A=x_A,X_B=x_B]\log \sum_{y}\frac{\hat{\mathbf{p}}_{x_A}(y) {\mathbf{p}}_{x_B}(y)}{\Pr[Y=y]} \Pr[X_B=x_B]-C\\
     =& \sum_{x_A,x_B} \Pr[X_A=x_A]\Pr[X_B=x_B|X_A=x_A]\log \sum_{y}\frac{\hat{\mathbf{p}}_{x_A}(y) {\mathbf{p}}_{x_B}(y)}{\Pr[Y=y]} \Pr[X_B=x_B]-C
\end{align*}

Moreover, fixing $X_A=x_A$

\begin{align*}
    &\sum_{x_B}\sum_{y}\frac{\hat{\mathbf{p}}_{x_A}(y) {\mathbf{p}}_{x_B}(y)}{\Pr[Y=y]} \Pr[X_B=x_B]\\
    =& \sum_{x_B}\sum_{y}\frac{\hat{\mathbf{p}}_{x_A}(y) \Pr[X_B=x_B,Y=y]}{\Pr[Y=y]} \\
    =& \sum_{x_B}\sum_{y}{\hat{\mathbf{p}}_{x_A}(y) \Pr[X_B=x_B|Y=y]} \\
    =& \sum_{y}\hat{\mathbf{p}}_{x_A}(y)=1
\end{align*}

Thus, $\sum_{y}\frac{\hat{\mathbf{p}}_{x_A}(y) {\mathbf{p}}_{x_B}(y)}{\Pr[Y=y]} \Pr[X_B=x_B]$ can be seen as a forecast for $X_B=x_B$. Since $LSR(\mathbf{p},\mathbf{q})=\sum_{\sigma}\mathbf{p}(\sigma)\log \mathbf{q}(\sigma)\leq \sum_{\sigma}\mathbf{p}(\sigma)\log \mathbf{p}(\sigma)=LSR(\mathbf{p},\mathbf{p})$ for any $\mathbf{p},\mathbf{q}\in\Delta_{\Sigma}$, we have

    \begin{align*} \numberthis \label{eq:truthful}
        & \sum_{x_A,x_B} \Pr[X_A=x_A]\Pr[X_B=x_B|X_A=x_A]\log \sum_{y}\frac{\hat{\mathbf{p}}_{x_A}(y) {\mathbf{p}}_{x_B}(y)}{\Pr[Y=y]} \Pr[X_B=x_B]-C\\ 
        \leq & \sum_{x_A,x_B} \Pr[X_A=x_A]\Pr[X_B=x_B|X_A=x_A]\log \Pr[X_B=x_B|X_A=x_A]-C\\
        =& \sum_{x_A,x_B} \Pr[X_A=x_A]\Pr[X_B=x_B|X_A=x_A]\log \Pr[X_B=x_B|X_A=x_A]\\
        &-\sum_{x_A,x_B} \Pr[X_A=x_A,X_B=x_B]\log \Pr[X_B=x_B]\\
        &= \sum_{x_A,x_B} \Pr[X_A=x_A,X_B=x_B]\log \frac{\Pr[X_B=x_B|X_A=x_A]}{\Pr[X_B=x_B]}\\
        =&I(X_A;X_B)
    \end{align*}
    
The non-negativity of the Shannon mutual information implies that agents are willing to participate in the mechanism.

It remains to analyze the strictness of the truthfulness. We need to show for any $x_A$, given that Alice receives $X_A=x_A$, she will obtain strictly less payment via reporting $\hat{\mathbf{p}}_{x_A}\neq \mathbf{p}_{x_A}$. 

Given that Alice receives $X_A=x_A$, her expected payment is 
\begin{align*} \tag{see equation (\ref{eq:truthful})}
    & \sum_{x_B} \Pr[X_B=x_B|X_A=x_A]\log \sum_{y}\frac{\hat{\mathbf{p}}_{x_A}(y) {\mathbf{p}}_{x_B}(y)}{\Pr[Y=y]} \Pr[X_B=x_B]-C\\ \numberthis \label{eq:strict}
    \leq & \sum_{x_B}\Pr[X_B=x_B|X_A=x_A]\log \Pr[X_B=x_B|X_A=x_A]-C
\end{align*}  

Note that $\sum_{\sigma}\mathbf{p}(\sigma)\log \mathbf{q}(\sigma)< \sum_{\sigma}\mathbf{p}(\sigma)\log \mathbf{p}(\sigma)$ when $\mathbf{q}\neq \mathbf{p}$. When the prior is stable, since $\hat{\mathbf{p}}_{x_A}\neq \mathbf{p}_{x_A}$, then $\mathbf{p}_{x_B},\hat{\mathbf{p}}_{x_A},(\Pr[Y=y])_y$ is not the solution of system (\ref{soe}). This implies that there exists $x_B$ such that 
$$\Pr[X_B=x_B|X_A=x_A]\neq \sum_{y}\frac{\hat{\mathbf{p}}_{x_A}(y) {\mathbf{p}}_{x_B}(y)}{\Pr[Y=y]} \Pr[X_B=x_B]. $$ Thus, the inequality (\ref{eq:strict}) must be strict. Therefore, when the prior is stable, the common ground mechanism is strictly truthful.

\end{proof}

\section{$PS$-gain}\label{sec:psgain}
In this section, we will extend the maximum likelihood estimator method in \citet{raykar2010learning} to a general family of optimization goals---$PS$-gain and compare the general family with our $f$-mutual information gain. We will see the application of $PS$-gain requires either one of the information sources to be low dimensional or that we have a simple generative model for the distribution over one of the information sources and ground truth label. Thus, the range of applications of $PS$-gain is more limited compared with the range of applications of $f$-mutual information gain. 

In \citet{raykar2010learning}, $X_A$ is a feature vector which has multiple crowdsourced labels $X_B$. We have access to $(x_A^{\ell},x_B^{\ell})_{\ell\in\mathcal{L}}$ which are i.i.d samples of $(X_A,X_B)$. \citet{raykar2010learning} also have the conditional independence assumption. 

\subsection{Maximum likelihood estimator (MLE)}

%Let $h_A:\Sigma_A\mapsto \Delta_{\Sigma}$ be the hypothesis that maps $X_A=x_A$ to a forecast vector for ground truth label $Y$, where $\Theta_A$ is the parameter that we are aiming to learn. Let $v_B:\Sigma_A\mapsto [0,1]^{|\Sigma|}$ be the mapping that maps $X_B=x_B$ to a vector in $[0,1]^{|\Sigma|}$, where $\Theta_B$ is the parameter that we are aiming to learn. Note here $v_B$ does not map the $X_B=x_B$ to a forecast vector. 

Let $\Theta_A,\Theta_B$ be two parameters that control the distribution over $X_A$ and $Y$ and the distribution over $X_B$ and $Y$ respectively. 

With the conditional independence assumption, we have 
\begin{align*}
   \log \Pr[(x_A^{\ell},x_B^{\ell})_{\ell\in\mathcal{L}}|\Theta_A,\Theta_B]=&\log \Pi_{\ell\in\mathcal{L}} \Pr[X_B=x_B^{\ell}|X_A=x_A^{\ell},\Theta_A,\Theta_B]\\
   =&\log \Pi_{\ell\in\mathcal{L}} \sum_y \Pr[X_B=x_B^{\ell}|Y=y,\Theta_B]\Pr[Y=y|X_A=x_A^{\ell},\Theta_A]\\
   =& \sum_{\ell\in\mathcal{L}}\log\bigg(\sum_y \Pr[X_B=x_B^{\ell}|Y=y,\Theta_B]\Pr[Y=y|X_A=x_A^{\ell},\Theta_A]\bigg)\\
   %=& \sum_{\ell\in\mathcal{L}}\log\bigg( v_B(x_B^{\ell})\cdot h_A(x_A^{\ell})\bigg)
\end{align*}

%Let $h_A:\Sigma_A\mapsto \Delta_{\Sigma}$ be the hypothesis that maps $X_A=x_A$ to a forecast vector $(\Pr[Y=y|X_A=x_A^{\ell},\Theta_A])_y$ for ground truth label $Y$, where $\Theta_A$ is the parameter that we are aiming to learn. Let $v_B:\Sigma_A\mapsto [0,1]^{|\Sigma|}$ be the mapping that maps $X_B=x_B$ to $(\Pr[X_B=x_B|Y=y,\Theta_B])_y$, where $\Theta_B$ is the parameter that we are aiming to learn. Note here $v_B$ does not map the $X_B=x_B$ to a forecast vector. 

%With the above definition, 

%\begin{align*}\log \Pr[(x_A^{\ell},x_B^{\ell})_{\ell\in\mathcal{L}}|\Theta_A,\Theta_B]= \sum_{\ell\in\mathcal{L}}\log\bigg( v_B(x_B^{\ell})\cdot h_A(x_A^{\ell})\bigg)\end{align*}

%We use $\mathbf{v}\cdot \mathbf{v}'$ to represent the dot product between two vectors.

The MLE is a pair of parameters $\Theta_A^*,\Theta_B^*$ that maximizes the expected $$\log \Pr[(x_A^{\ell},x_B^{\ell})_{\ell\in\mathcal{L}}|\Theta_A,\Theta_B]=\sum_{\ell\in\mathcal{L}}\log\bigg(\sum_y \Pr[X_B=x_B^{\ell}|Y=y,\Theta_B]\Pr[Y=y|X_A=x_A^{\ell},\Theta_A]\bigg).$$ \citet{raykar2010learning} use the MLE to estimate the parameters. In order to compare this MLE method with our $f$-mutual information gain framework, we map this MLE method into our language and provide a theoretical analysis for the condition when MLE is meaningful.

\paragraph{$LSR$-gain/MLE}

\begin{description}
\item[Hypothesis] We are given $\mathcal{H}_A=\{h_A:\Sigma_A\mapsto \Delta_{\Sigma}\}$, $\mathcal{V}_B=\{v_B:\Sigma_B\mapsto [0,1]^{|\Sigma|}\}$: the set of hypotheses candidates for $X_A$ and  $X_B$, respectively. Note that $v_B$ maps $x_B\in\Sigma_B$ into a vector in $[0,1]^{|\Sigma|}$ rather than a distribution vector. 

\item[Gain] 

We see $$(v_B(x_B) \cdot h_A(x_A^{\ell}))_{x_B}$$ as a forecast  for random variable $X_B$ conditioning on $X_A=x_A^{\ell}$ and we reward the hypotheses $LSR$-gain---the accuracy of this forecast via log scoring rule (LSR):

\begin{align*}
    \sum_{\ell\in\mathcal{L}}LSR\bigg(x_B^{\ell}, ( v_B(x_B) \cdot h_A(x_A^{\ell}))_{x_B}\bigg)
    =\sum_{\ell\in\mathcal{L}}\log\bigg( v_B(x_B^{\ell}) \cdot h_A(x_A^{\ell})\bigg)
\end{align*}
\end{description}

We use $\mathbf{v}\cdot \mathbf{v}'$ to represent the dot product between two vectors.

Note that by picking $\mathcal{H}_A$ as the set of mappings---associated with a set of parameters $\{\Theta_A\}$---that map $X_A=x_A$ to $(\Pr[Y=y|X_A=x_A^{\ell},\Theta_A])_y$ and picking $\mathcal{V}_B$ as the set of mappings---associated with a set of parameters $\{\Theta_B\}$---that map $X_B=x_B$ to $(\Pr[X_B=x_B|Y=y,\Theta_B])_y$, maximizing $LSR$-gain is equivalent to obtaining MLE. 

The idea of $LSR$-gain is very similar with the original peer prediction idea introduced in Section~\ref{sec:single} as well as our common ground mechanism. 

\begin{theorem}\label{thm:mle}
When $\sum_{x_B\in\Sigma_B} v_B(x_B)=(1,1,..,1)$ for all $v_B\in\mathcal{V}_B$, the ground truth $Y$ corresponds to a maximizer of $LSR$-gain: $$v_B^*(x_B)=(\Pr[X_B=x_B|Y=y])_y\qquad h_A^*(x_A)=(\Pr[Y=y|X_A=x_A])_y.$$ The maximum is the conditional Shannon entropy $H(X_B|X_A)$.
\end{theorem}

\begin{remark}
Note that without the restriction: $\sum_{x_B\in\Sigma_B} v_B(x_B)=(1,1,..,1)$ for all $v_B\in\mathcal{V}_B$, $$v_B^*(x_B)=(\Pr[X_B=x_B|Y=y])_y\qquad h_A^*(x_A)=(\Pr[Y=y|X_A=x_A])_y$$ is not a maximizer and we will have a meaningless maximizer $v_B(x_B)=(1,1,..,1),\forall x_B$ and $h_A(x_A)=(1,0,...,0),\forall x_A$. 
\end{remark}

By picking $\mathcal{V}_B$ as the set of mappings---associated with a set of parameters $\{\Theta_B\}$---that map $X_B=x_B$ to $(\Pr[X_B=x_B|Y=y,\Theta_B])_y$, the restriction $\sum_{x_B\in\Sigma_B} v_B(x_B)=(1,1,..,1)$ for all $v_B\in\mathcal{V}_B$ satisfies naturally. However, it requires the knowledge of the generative distribution model over $X_B$ and $Y$ with parameter $\Theta_B$. \citet{raykar2010learning} assume a simple distribution model between $X_B$ and $Y$ with parameter $\Theta_B$---conditioning the ground truth label, the crowdsourced feedback $X_B$ is drawn from a binomial distribution, such that $\Pr[X_B=x_B|Y=y,\Theta_B]$ has a simple explicit form.

\begin{proof}[Proof of Theorem~\ref{thm:mle}]
    \begin{align*}
        &\E\sum_{\ell\in\mathcal{L}}\log\bigg(v_B(x_B^{\ell}) \cdot h_A(x_A^{\ell})\bigg)\\
        =&\sum_{x_A\in\Sigma_A,x_B\in\Sigma_B}\Pr[X_A=x_A,X_B=x_B] \log \bigg(v_B(x_B) \cdot h_A(x_A)\bigg)\\
        =&\sum_{x_A\in\Sigma_A,x_B\in\Sigma_B}\Pr[X_A=x_A]\Pr[X_B=x_B|X_A=x_A] \log \bigg(v_B(x_B) \cdot h_A(x_A)\bigg)\\
        =& \sum_{x_A\in\Sigma_A,x_B\in\Sigma_B}\Pr[X_A=x_A]LSR\bigg(( \Pr[X_B=x_B|X_A=x_A])_{x_B},(v_B(x_B) \cdot h_A(x_A))_{x_B}\bigg)
    \end{align*}

Fixing $X_A=x_A$, since $\sum_{x_B\in\Sigma_B} v_B(x_B)=(1,1,...,1)$ for all $v_B\in\mathcal{V}_B$, we have

\begin{align*}
   \sum_{x_B} \bigg(v_B(x_B) \cdot h_A(x_A)\bigg)=\sum_y h_A(x_A)(y)=1
\end{align*}
Since $LSR(\mathbf{p},\mathbf{q})\leq LSR(\mathbf{p},\mathbf{p})$ for any $\mathbf{p},\mathbf{q}\in\Delta_{\Sigma}$, we have

\begin{align*}
    &\E\sum_{\ell\in\mathcal{L}}\log\bigg(v_B(x_B) \cdot h_A(x_A)\bigg)\\
    =&\sum_{x_A\in\Sigma_A,x_B\in\Sigma_B}\Pr[X_A=x_A]LSR\bigg(( \Pr[X_B=x_B|X_A=x_A])_{x_B},(v_B(x_B) \cdot h_A(x_A))_{x_B}\bigg)\\
    \leq&\sum_{x_A\in\Sigma_A,x_B\in\Sigma_B}\Pr[X_A=x_A]LSR\bigg(( \Pr[X_B=x_B|X_A=x_A])_{x_B},( \Pr[X_B=x_B|X_A=x_A])_{x_B}\bigg)\\
    =& \sum_{x_A\in\Sigma_A,x_B\in\Sigma_B}\Pr[X_A=x_A]\Pr[X_B=x_B|X_A=x_A] \log \Pr[X_B=x_B|X_A=x_A]\\
    =& H(X_B|X_A)\\ \tag{conditional independence}
    =&\sum_{x_A\in\Sigma_A,x_B\in\Sigma_B}\Pr[X_A=x_A]\Pr[X_B=x_B|X_A=x_A] \log\bigg(\sum_y \Pr[X_B=x_B|Y=y]\Pr[Y=y|X_A=x_A]\bigg)
\end{align*}

Thus, $$v_B^*(x_B)=(\Pr[X_B=x_B|Y=y])_y\qquad h_A^*(x_A)=(\Pr[Y=y|X_A=x_A])_y$$ is a maximizer and the maximum is the conditional Shannon entropy $H(X_B|X_A)$.

\end{proof}

\subsection{Extending $LSR$-gain to $PS$-gain}

The property $LSR(\mathbf{p},\mathbf{q})=\sum_{\sigma}\mathbf{p}(\sigma)\log \mathbf{q}(\sigma)\leq \sum_{\sigma}\mathbf{p}(\sigma)\log \mathbf{p}(\sigma)=LSR(\mathbf{p},\mathbf{p})$ for any $\mathbf{p},\mathbf{q}\in\Delta_{\Sigma}$ is also valid for all proper scoring rules. Thus, we can naturally extend the MLE to $PS$-gain by replacing the $LSR$ by any given proper scoring rule $PS$. 

\paragraph{$PS$-gain}
\begin{description}
\item[Hypothesis] We are given $\mathcal{H}_A=\{h_A:\Sigma_A\mapsto \Delta_{\Sigma}\}$, $\mathcal{V}_B=\{v_B:\Sigma_B\mapsto [0,1]^{|\Sigma|}\}$: the set of hypotheses candidates for $X_A$ and  $X_B$, respectively. %Note that $v_B$ maps $x_B\in\Sigma_B$ into a vector in $[0,1]^{|\Sigma|}$ rather than a distribution vector. 

\item[Gain] 

We see $$(v_B(x_B) \cdot h_A(x_A^{\ell}))_{x_B}$$ as a forecast  for random variable $X_B$ conditioning on $X_A=x_A^{\ell}$ and we reward the hypotheses $PS$-gain---the accuracy of this forecast via a given proper scoring rule $PS$:

\begin{align*}
    \sum_{\ell\in\mathcal{L}}PS\bigg(x_B^{\ell}, ( v_B(x_B) \cdot h_A(x_A^{\ell}))_{x_B}\bigg)
\end{align*}
\end{description}

Note that the general $PS$-gain may involve the calculations of $( v_B(x_B) \cdot h_A(x_A^{\ell}))_{x_B}$ while $LSR$-gain only requires the value of $v_B(x_B^{\ell}) \cdot h_A(x_A^{\ell})$. Thus, unlike $LSR$-gain, the general $PS$-gain may be only applicable for low dimensional $X_B$, even if we assume a simple generative distribution model over $X_B$ and $Y$.

\begin{theorem}
Given a proper scoring rule $PS$, when $\sum_{x_B\in\Sigma_B} v_B(x_B)=(1,1,...,1)$ for all $v_B\in\mathcal{V}_B$, the ground truth $Y$ corresponds to a $PS$-gain maximizer: $$v_B^*(x_B)=(\Pr[X_B=x_B|Y=y])_y\qquad h_A^*(x_A)=(\Pr[Y=y|X_A=x_A])_y.$$ 
\end{theorem}

The proof is the same with Theorem~\ref{thm:mle} except that we replace $LSR(\mathbf{p},\mathbf{q})\leq LSR(\mathbf{p},\mathbf{p})$ by $PS(\mathbf{p},\mathbf{q})\leq PS(\mathbf{p},\mathbf{p})$ for any $\mathbf{p},\mathbf{q}\in\Delta_{\Sigma}$.

%Note that the design of $PS$-gain is very similar with the original peer prediction idea introduced in Section~\ref{sec:single}. 

\subsection{Comparing $PS$-gain with $f$-mutual information gain}\label{sec:comparison}

Generally, $f$-mutual information gain can be applied to a more general setting.  

$PS$-gain requires the restriction $\sum_{x_B\in\Sigma_B} v_B(x_B)=(1,1,...,1)$ for all $v_B\in\mathcal{V}_B$. Thus, $PS$-gain requires the full knowledge of $v_B$ for all $v_B\in\mathcal{V}_B$ to check whether it satisfies the restriction, while for the $f$-mutual information gain, it is sufficient to just have the access to the outputs of the hypothesis: $\{h_B(x_B^{\ell})\}_{\ell\in\mathcal{L}_B}$. Therefore, in the mechanism design part, we can only use $f$-mutual information gain to design focal mechanisms since we only have the outputs from agents. 

Moreover, $\sum_{x_B\in\Sigma_B} v_B(x_B)=(1,1,...,1)$ is also hard to check when $|\Sigma_B|$ is very large. For example, when $x_B$ is a $100\times 100$ black-and-white image, $|\Sigma_B|=2^{100}$ and checking $\sum_{x_B\in\Sigma_B} v_B(x_B)=(1,1,...,1)$ requires $2^{100}$ time. Normalizing $v_B$ such that it satisfies the condition also requires $2^{100}$ time. Thus, when $|\Sigma_B|$ is very large, we need a simple generative distribution model between $X_B$ and $Y$ with parameter $\Theta_B$ such that we can pick $\mathcal{V}_B$ as the set of mappings---associated with a set of parameters $\{\Theta_B\}$---that map $X_B=x_B$ to $(\Pr[X_B=x_B|Y=y,\Theta_B])_y$, to make the restriction $\sum_{x_B\in\Sigma_B} v_B(x_B)=(1,1,..,1)$ for all $v_B\in\mathcal{V}_B$ satisfy naturally. When we have the simple generative distribution model, we can use $LSR$-gain. The general $PS$-gain involves the calculations of the $|\Sigma_B|$ dimensional vector---$(v_B(x_B) \cdot h_A(x_A^{\ell}))_{x_B}$---for each $x_A^{\ell}$. Thus, the general $PS$-gain is only applicable to low dimensional $X_B$.

In the learning with noisy labels problem, the distribution between $X_B$ and $Y$ can be represented by a simple transition matrix and $X_B$ is low dimensional. Therefore, both $PS$-gain and $f$-mutual information gain can be applied to the learning with noisy labels problem.

Therefore, the application of $PS$-gain requires either one of the information sources to be low dimensional or that we have a simple generative model for the distribution over one of the information sources and ground truth label, while $f$-mutual information gain does not have the restrictions.

%Moreover, given the prior of $Y$, we can still use the $f$-mutual information gain by fixing $\mathbf{p}$ in $MIG^f(h_A,h_B,\mathbf{p})$ as the prior of $Y$ while it's not natural to incorporate this information into the maximization of $PS$-gain. 

%The range of application of $f$-mutual information gain is wide. We do not need any distribution model and can make both $h_A$ and $h_B$ neural networks and learn them simultaneously using $f$-mutual information gain. 

\section{Conclusion and discussion}\label{sec:discussions} 
We build a natural connection between mechanism design and machine learning by addressing two related problems: (1) co-training: learning to forecast ground truth using two conditionally independent sources, without access to labeled data; (2) forecast elicitation: eliciting high quality forecasts from the crowds without verification, by the same information theoretic approach. 

For the co-training problem, as usual in the related literature, we reduce the problem to an optimization problem and do not investigate the computation complexity or the data requirements. To implement our $f$-mutual information gain framework in practice, we implicitly assume that for high dimensional $X_A,X_B$, there exists a trainable set of hypotheses (e.g. neural networks) that is sufficiently rich to contain the Bayesian posterior predictor but not everything to cause over-fitting. The most apparent empirical direction will be running experiments on real data by training two neural networks to test our algorithms. Interesting theoretic directions include the analysis of the Bayesian risk and the influence of the choice of the convex function $f$ on the convergence rate.

For forecast elicitation, the most apparent direction will be performing real-world experiments. To apply our mechanisms, we do not need that every two agents' information is conditionally independent. In fact, for each agent, we only need to find a single reference agent for her such that the reference agent's information is conditionally independent of hers. Then we can run our mechanisms on the agent and her reference agent. In practice, we can pair the agents with some side information and make sure each pair of agents' information is conditionally independent. \gs{Can we cite paper 4 here?}

Another interesting direction is to ensure fairness, in particular, that agents are not incentivized to coordinate on stereotypes. One solution, is suppressing information from some of the agents and using our framework. However, when this is not possible, the prior peer prediction work on cheap signals~\cite{2016arXiv160607042G,2018arXiv180208312K} may be helpful in addressing this issue.  

\section*{Acknowledgement}
We thank Clayton Scott for useful conversations. 

\bibliographystyle{ACM-Reference-Format}
\bibliography{ref}

%\fullversion

\appendix 

\section{Additional proof(s)}

{
\renewcommand{\thetheorem}{\ref{claim:ci}}

\begin{claim}
When random variables $X_A$, $X_B$ are independent conditioning on $Y$, 
\begin{align*}
    K(X_A=x_A,X_B=x_B)
    =&\sum_y {\Pr[Y=y]}K(X_A=x_A,Y=y) K(X_B=x_B,Y=y)\\
    =&\sum_y \Pr[Y=y|X_A=x_A] K(X_B=x_B,Y=y)\\
    =&\sum_y \frac{\Pr[Y=y|X_A=x_A]\Pr[Y=y|X_B=x_B]}{\Pr[Y=y]}.
\end{align*}
\end{claim}

\addtocounter{theorem}{-1}
}

\begin{proof}
    \begin{align*}
        K(X_A=x_A,X_B=x_B)=&\frac{\Pr[X_A=x_A,X_B=x_B]}{\Pr[X_A=x_A]\Pr[X_B=x_B]}\\
        =& \frac{\sum_y \Pr[Y=y]\Pr[X_A=x_A,X_B=x_B|Y=y]}{\Pr[X_A=x_A]\Pr[X_B=x_B]}\\ \tag{Conditional independence}
        =& \frac{\sum_y \Pr[Y=y]\Pr[X_A=x_A|Y=y]\Pr[X_B=x_B|Y=y]}{\Pr[X_A=x_A]\Pr[X_B=x_B]}\\ \tag{PMI=posterior/prior}
        =& \sum_y {\Pr[Y=y]}K(X_A=x_A,Y=y) K(X_B=x_B,Y=y)\\
        =&\sum_y \Pr[Y=y|X_A=x_A] K(X_B=x_B,Y=y)\\
    =&\sum_y \frac{\Pr[Y=y|X_A=x_A]\Pr[Y=y|X_B=x_B]}{\Pr[Y=y]}.
    \end{align*}
\end{proof}

{
\renewcommand{\thetheorem}{\ref{thm:focal}}

\begin{theorem}
Given the prior distribution over the $Y$, with the conditional independence assumption, with a priori similar and random order assumption, when $\max\{|\mathcal{L}_A|,|\mathcal{L}_B|\}\geq 2$ and the prior is stable and well-defined, when the convex function $f$ is differentiable and $f'$ is invertible, $MCG(f)$ is focal.

When both Alice and Bob are honest, each of them's expected payment in $MCG(f)$ is
$$ MI^f(X_A;X_B). $$
\end{theorem}

\addtocounter{theorem}{-1}
}

\begin{proof}
Given that Alice's strategy is $s_A$ and Bob's strategy is $s_B$, with the a priori similar and random order assumption, we represent agents' report as the output (possibly being random) of their strategy operating on the private information.

We start to show $MCG(f)$ is strictly truthful. Given that Alice is honest, based on Lemma~\ref{lem:focal}, Bob will maximize his expected payment if and only if $\forall \ell_1,\ell_2$,

\begin{align*}
    R({\mathbf{p}}_{{x_A}^{\ell_1}},s_B(x_B^{\ell_2}))=f'(K(x_A^{\ell_1},x_B^{\ell_2})).
\end{align*}

Note that in $MCG(f)$, 

\begin{align*}
    R({\mathbf{p}}_{{x_A}^{\ell_1}},s_B(x_B^{\ell_2}))=f'(\sum_y \frac{\mathbf{p}_{{x_A}^{\ell_1}}(y)s_B(x_B^{\ell_2})(y)}{\Pr[Y=y]})
\end{align*}

Since the prior is stable, the above equation is satisfied for all possible ${x_A}^{\ell_1}$ if and only if Bob tells the truth, i.e., reporting ${\mathbf{p}}_{{x_B}^{\ell_2}}$. Therefore, $MCG(f)$ is strictly truthful. 

It remains to show $MCG(f)$ pays truth-telling the most and strictly better than any other non-permutation strategy profile. When agents maximize the expected payment, 
\begin{align*}
    R(s_A({x_A}^{\ell_1}),s_B(x_B^{\ell_2}))=f'(K(x_A^{\ell_1},x_B^{\ell_2})).
\end{align*}

Recall that we defined 
\begin{align*}
    R(s_A({x_A}^{\ell_1}),s_B(x_B^{\ell_2}))=f'(\sum_y \frac{s_A({x_A}^{\ell_1})(y)s_B(x_B^{\ell_2})(y)}{\Pr[Y=y]}).
\end{align*}

Thus, when $f'$ is invertible, we have 

\begin{align*}
    \sum_y \frac{s_A({x_A}^{\ell_1})(y)s_B(x_B^{\ell_2})(y)}{\Pr[Y=y]}=K(x_A^{\ell_1},x_B^{\ell_2})
\end{align*} for any $x_A^{\ell_1},x_B^{\ell_2}$. This is exactly system (\ref{soe}).

With the conditional independence assumption, when agents tell the truth, the above system will be satisfied. Therefore, agents can maximize their expected payment via truth-telling. The non-negativity of $MI^f$ implies that agents are willing to participate in the mechanism. 

Moreover, when the prior is well-defined, if the prior $\Pr[Y]$ is a uniform distribution, then any permutation strategy profile can solve the above system and as well as maximize agents' expected payment. Even if the prior $\Pr[Y]$ is not a uniform distribution, although not all permutation strategy profiles solve the above system, still any solution of the above system must correspond to a permutation strategy profile, given the prior is well-defined. Therefore, when agents maximize their expected payment, their strategy profile must be a permutation strategy profile or truth-telling, which implies $MCG(f)$ is focal.

\end{proof}

{
\renewcommand{\thetheorem}{\ref{lem:focal}}

\begin{lemma}%\label{lem:focal}
With the conditional independence assumption, the expected total payment is maximized over Alice and Bob's strategies if and only if $\forall \ell_1 \in \mathcal{L}_A, \ell_2 \in \mathcal{L}_B$, for any $(x_A^{\ell_1},x_B^{\ell_2})\in\Sigma_A\times\Sigma_B$, $$R(\hat{\mathbf{p}}_{{x_A}^{\ell_1}}^{\ell_1},\hat{\mathbf{p}}_{{x_B}^{\ell_2}}^{\ell_2})=f'(K(x_A^{\ell_1},x_B^{\ell_2})).$$ The maximum is $$MI^f(X_A;X_B).$$
\end{lemma}

\addtocounter{theorem}{-1}
}

\begin{proof}
    Without loss of generality, it is sufficient to analyze Alice's strategy and report. With the a priori similar and random order assumption, $\hat{\mathbf{p}}_{{x_A}^{\ell_1}}^{\ell_1}$ can be represented as $s_A({x_A}^{\ell_1})$ since the index of the task $\ell_1$ is meaningless to Alice when all tasks appear in a random order, independently drawn for each agent.  The strategy can be seen as a random predictor. Thus,  we can use the same proof of Lemma~\ref{lem:pplearn} to prove Lemma~\ref{lem:focal}.
\end{proof}

%\bibliographystyle{ACM-Reference-Format}
%\bibliography{}

\end{document}